\documentclass[journal]{IEEEtran}
\ifCLASSINFOpdf
\else
\fi
\usepackage{graphicx}
\usepackage{subfigure}
\usepackage{epstopdf}
\usepackage{amssymb,amsmath,amsthm,amsfonts}
\usepackage{algorithmicx}
\usepackage{algorithm}
\usepackage{algpseudocode}
\usepackage{url}
\usepackage{cite}
\newtheorem{theorem}{Theorem}
\newtheorem{assumption}{Assumption}
\newtheorem{lemma}{Lemma}
\newtheorem{corollary}{Corollary}

\newcommand{\E}{\mathbb{E}}
\newcommand{\R}{\mathbb{R}}
\DeclareMathOperator*{\argmin}{arg\,min}

\begin{document}
%
% paper title
% Titles are generally capitalized except for words such as a, an, and, as,
% at, but, by, for, in, nor, of, on, or, the, to and up, which are usually
% not capitalized unless they are the first or last word of the title.
% Linebreaks \\ can be used within to get better formatting as desired.
% Do not put math or special symbols in the title.
\title{Accelerating Mini-batch SARAH by Step Size Rules }
%
%
% author names and IEEE memberships
% note positions of commas and nonbreaking spaces ( ~ ) LaTeX will not break
% a structure at a ~ so this keeps an author's name from being broken across
% two lines.
% use \thanks{} to gain access to the first footnote area
% a separate \thanks must be used for each paragraph as LaTeX2e's \thanks
% was not built to handle multiple paragraphs
%

\author{Zhuang~Yang,
         Zengping~Chen,
        ~and~Cheng~Wang,~\IEEEmembership{Senior Member,~IEEE,}
        %Yu~Zang,%~\IEEEmembership{Life~Fellow,~IEEE}
        %~and~Jonathan~Li,~\IEEEmembership{Senior Member,~IEEE}% <-this % stops a space
\thanks{(Corresponding authors: Zengping Chen). Z. Yang, Z. P. Chen are with the School of Electronics and Communication Engineering, Sun Yat-sen University, Guangzhou510275, China(e-mail: zhuangyng@163.com, chenzengp@mail.sysu.edu.cn).}
%
%C. Wang, Y. Zang and J. Li are with the Fujian Key Laboratory of Sensing and Computing for Smart Cities, School of Information Science and Engineering, Xiamen University, Xiamen, FJ 361005, China (e-mail: zhuangyng@163.com, , zangyu7@126.com, junli@xmu.edu.cn ) .}% <-this % stops a space
\thanks{C. Wang is with the Fujian Key Laboratory of Sensing and Computing for Smart Cities, School of Information Science and Engineering, Xiamen University, Xiamen, FJ 361005, China(email: cwang@xmu.edu.cn}% <-this % stops a space
\thanks{Manuscript received July 25, 2016.}}

\maketitle

% As a general rule, do not put math, special symbols or citations
% in the abstract or keywords.
\begin{abstract}
StochAstic Recursive grAdient algoritHm (SARAH), originally proposed for convex optimization and also proven to be effective for general nonconvex optimization, has received great attention due to its simple recursive framework for updating stochastic gradient estimates. The performance of SARAH significantly depends on the choice of step size sequence. However, SARAH and its variants often employ a best-tuned step size by mentor, which is time consuming in practice. Motivated by this gap, we proposed a variant of the Barzilai-Borwein (BB) method, referred to as the Random Barzilai-Borwein (RBB) method, to calculate step size for SARAH in the mini-batch setting, thereby leading to a new SARAH method: MB-SARAH-RBB. We prove that MB-SARAH-RBB converges linearly in expectation for strongly convex objective functions. We analyze the complexity of MB-SARAH-RBB and show that it is better than the original method. Numerical experiments on standard data sets indicate that MB-SARAH-RBB outperforms or matches state-of-the-art algorithms.
\end{abstract}

% Note that keywords are not normally used for peerreview papers.

\section{Introduction}
\label{introduction}
\IEEEPARstart{S}{tochastic} gradient descent (SGD) type methods have been a core methodology in applications to large scale problems in machine learning and related areas \cite{mairal2009online,zhang2004solving,felzenszwalb2009object,blanz2003face,hoffman2010online,bottou2010large,lin2011large,zhang2014facial,tao2014stochastic,bugallo2017adaptive,du2017learning,li2018preconditioned}. The classical SGD method only requires a single random example per-iteration to approximate the full gradient. Such strategy usually makes SGD perform with a low computational complexity per-iteration. While SGD makes rapid progress early on, the convergence rate of SGD is significantly deteriorated by the intrinsic variance of its stochastic estimator. Even for strongly and smooth problem, SGD only converges sub-linearly \cite{moulines2011non-asymptotic}.

Traditionally, there are three common ways to decrease the variance caused by the stochastic estimate. The first one is taking a decreasing step size sequence \cite{nemirovski2009robust,bottou2018optimization}. However, this will further reduce the convergence rate. Moreover, it's known that the practical convergence of SGD is very sensitive to the choice of step size sequence, which needs to be hand-picked. A second approach is using a mini-batching technique \cite{cotter2011better,konevcny2016mini}. Obviously, this requires more computations. The last method is using important sampling strategy \cite{needell2014stochastic,csiba2018importance}. Although effective, this technique is not always practical, as the computation of the sampling mechanism relates to the dimensionality of model parameters \cite{fu2017cpsg}. In summary, any variance reduction technique does not come for free.

In recent years, some advanced stochastic variance-reduced algorithms have emerged, which use the specific form of objective function and
combine some deterministic and stochastic aspects to reduce variance of the steps. The popular examples of these methods are the stochastic average gradient (SAG) method\cite{Roux2012}, the SAGA method \cite{defazio2014saga}, the stochastic dual coordinate ascent (SDCA) method \cite{shalev2013stochastic}, the stochastic variance reduced gradient (SVRG) method \cite{johnson2013accelerating}, the accelerated mini-batch Prox-SVRG (Acc-Prox-SVRG) method \cite{nitanda2014stochastic}, the mini-batch semi-stochastic gradient descent (mS2GD) method \cite{konevcny2016mini}, the StochAstic Recursive grAdient algoritHm (SARAH) \cite{nguyen2017sarah} and the Stochastic Path-Integrated Differential EstimatoR method (SPIDER) \cite{fang2018spider}, all of which have faster convergence rate than that of SGD. Specifically, these methods work with a fixed step size. However, the step size is often chosen by mentor. Hence, this is time consuming in practice.

More recently, SARAH, originally proposed for convex optimization, is gaining tremendous popularity due to only requiring a simple framework for updating stochastic gradient estimates \cite{nguyen2019finite}. Moreover, SARAH has been proven to be effective for general nonconvex optimization \cite{nguyen2017stochastic,nguyen2018inexact,horvath2018nonconvex,zhou2019stochastic,pham2019proxsarah,nguyen2019optimal}. Actually, SARAH and SVRG \cite{johnson2013accelerating} are two similar methods which  perform a deterministic step often called outer loop, where the full gradient of the objective functions was calculated at the outer loop, then followed by stochastic processes. The only difference between SVRG and SARAH is how the iterative scheme is performed in the inner loop. In addition, SARAH is a recursive method as SAGA \cite{defazio2014saga}, but do not store gradients as SAGA. Especially, different from SVRG and other methods (e.g., SAG, SDCA, mS2GD, etc.), SARAH  does not take an estimator that is unbiased in the last step. Instead, it is unbiased over a long history of the method. Specifically, the advantage of SARAH is that the iterative scheme of the inner loop itself can converge sub-linearly \cite{nguyen2017sarah}.

 Although Nguyen et al. \cite{nguyen2017sarah} pointed out that SARAH uses a large constant step size than that of SVRG, the step size is still chosen by mentor. Including the variants of SARAH also employed a constant step size \cite{nguyen2018inexact,nguyen2019optimal}. In addition, Pham et al. \cite{pham2019proxsarah} proposed proximal SARAH (ProxSARAH) for stochastic composite nonconvex optimization and showed that ProxSARAH works with new constant and adaptive step sizes, where the constant step size is much larger than
existing methods, including proximal SVRG (ProxSVRG) schemes \cite{xiao2014prox} in the single sample case and adaptive step-sizes are increasing along the inner iterations rather than diminishing as in stochastic proximal gradient descent methods. However, it is complicated to compute adaptive step size for ProxSARAH. Especially, ProxSARAH needs to control two step size sequences, which make it difficult to use in practice.

To deal with this demerit associated with SARAH, we propose using the random Barzilai-Borwein (RBB) method to automatically calculate step size for the mini-batch version of SARAH (MB-SARAH), proposed by Nguyen et al. \cite{nguyen2017stochastic} for nonconvex optimization, thereby leading to a new SARAH method named as MB-SARAH-RBB. The RBB method, a variant of the Barzilai-Borwein (BB) method \cite{barzilai1988two}, has been proposed by Yang et al. \cite{YANGrandom} and use it to calculate step size for mini-batch algorithms. However, they just discussed the choice of step size of SVRG-type algorithms, i.e., mS2GD and Acc-Prox-SVRG.

 The key contributions of this work are as follows:
\begin{itemize}
\item[] \textrm{1}) We propose to use the RBB method to compute step size for MB-SARAH and obtain a new SARAH method named as MB-SARAH-RBB. Unlike the work in \cite{YANGrandom}, when using the RBB method to calculate step size, we multiply a constant parameter, which is pivotal to ensure the convergence of MB-SARAH-RBB.
\item[] \textrm{2})  We prove the convergence of our MB-SARAH-RBB method and show that its complexity is better than SARAH in the mini-batch setting.
\item[] \textrm{3}) We conduct experiments for MB-SARAH-RBB on solving logistic regression problem. Experimental results on three benchmark data sets show that the proposed method outperforms or matches state-of-the-art algorithms.
\end{itemize}

The rest of this paper is organized as follows. Section \ref{related work} discusses related works that are relevant to this paper. Section \ref{PFB} presents problem formulation and background. Section \ref{our method} proposes our MB-SARAH-RBB method. Section \ref{cona} presents the convergence analysis of MB-SARAH-RBB for strongly convex objective functions and discusses its complexity. Section \ref{experiments} conducts some empirical comparisons over some state-of-the-art approaches. Section \ref{conclusion} concludes the paper.

Notations: Throughout this paper, we view vectors as columns, and use $w^T$ to denote the transpose of a vector $w$. We use the symbol, $\|\cdot\|$, to denote the Euclidean vector norm, i.e., $\|w\|=\sqrt{w^Tw}$. We use $\E[Z]$ to denote the expectation of a random variable $Z$.

\section{Related Work}
\label{related work}
Early works that compute step sizes adaptively for SGD are based on (i) a function of the errors in the predictions or estimates, or (ii) a function of the gradient of the error measure. For example, Kesten \cite{kesten1958accelerated} pointed out that when consecutive errors in the estimate of the value of a parameter obtained by the Robbins-Monro procedure  \cite{robbins1951stochastic} are of opposite signs, the estimate is in the vicinity of the true value and accordingly the step size ought to be reduced. Further, an alternative version of the gradient adaptive step size algorithm within a stochastic approximation formulation was presented by Benveniste et al. \cite{Benveniste1990Adaptive}. In addition, RMSprop, propounded by Tieleman et al. \cite{tieleman2012divide}, adapts a step size per weight based on the observed sign changes in the gradients. For more related methods, we refer readers to \cite{george2006adaptive,yang2019mini} and references therein.

Recently, due to its simplicity and numerical efficiency, many researchers try to incorporate the BB method and its variants into SGD. For instance, Sopy{\l}a et al. \cite{sopyla2015stochastic} presented several variants of the BB method for SGD to train the linear SVM. Tan et al. \cite{tan2016barzilai} used the BB method to calculate the step size for SGD and SVRG, thereby putting forward two new approaches: SGD-BB and SVRG-BB. Moreover, they showed that SVRG-BB have linear convergence for strongly convex objective functions. To further accelerate the convergence rate of SVRG-BB, mS2GD-BB, incorporating the BB method into mS2GD (a variant of SVRG), was proposed by Yang et al. \cite{yang2018mini}. They presented that mS2GD-BB has linear convergence in expectation for nonsmooth strongly convex objective functions. In addition, Yang et al. \cite{yang2019accelerated} introduced the BB method into accelerated stochastic gradient (ASGD) methods and obtained a series of new ASGD methods. Moreover, for their proposed methods, they finished the proof of the convergence analysis and pointed out that the complexity of their proposed methods achieves the same level as the best known stochastic gradient methods. Further, when considering a ``big batch" for SGD, De et al. \cite{De2017Automated} introduced the backtracking line search and BB methods into SGD to calculate step size. Moreover, they pointed out that the performance of SGD, using an adaptive step size method based on the BB method, is better than that of using the backtracking line search on a range of convex problems. To obtain online step size, Yang et al. \cite{YANGrandom} put forward the RBB method and incorporated it into mS2GD and Acc-Prox-SVRG, generating two new approaches: mS2GD-RBB and Acc-Prox-SVRG-RBB. To avert the denominator being close to zero when using the BB, or RBB methods,  the stabilized Barzilai-Borwein (SBB) step size was proposed by Ma et al. \cite{Ma2018Stochastic}. Especially, they introduced it into SVRG and obtained SVRG-SBB for dealing with the ordinal embedding problem. Moreover, they showed that the SVRG-SBB method converges with a rate, $O(\frac{1}{T})$, where $T$ is the total number of iterations.

In addition to the above-mentioned methods, other strategies of choosing step size were used in SGD. For instance, two adaptive step size schemes, referred to as a recursive step size stochastic approximation (RSA) scheme and a cascading step size stochastic approximation (CSA) scheme, were put forward by Yousefian et al. \cite{yousefian2012stochastic}. They also finished the proof of convergence analysis of two new iteration schemes for strongly convex differentiable stochastic optimization problems. In addition, Mahsereci et al. \cite{JMLR:v18:17-049} suggested performing line search for an estimated function, which is computed by a Gaussian process with random samples. An online step size can also be obtained by using a hypergradient descent, where can be found in \cite{baydin2018online}. To greatly reduce the dependence of the algorithm on initial parameters when using hypergradient, Yang et al. \cite{yang2019mini} introduced the online step size (OSS) into the the mini-batch nonconvex stochastic variance reduced gradient (MSVRG) method \cite{reddi2016stochastic} and obtain the MSVRG-OSS method. Moreover, they showed that MSVRG-OSS has linear converges for strongly convex objective functions. Especially, they pointed out that the MSVRG-OSS method also can be used to deal with nonconvex problems. Other different types of choosing step size for SGD, we refer readers to \cite{almeida1999parameter,george2006adaptive,schaul2013no,tieleman2017divide} and  references therein.

\section{Problem formulation and background}
\label{PFB}
We focus on the following problem
\begin{eqnarray}
\min \limits_{w \in \R ^{d}} P(w)=\frac{1}{n} \sum_{i=1}^n f_{i}(w). \label{eq1-1}
\end{eqnarray}
where $n$ is the sample size, and each $f_i(w): \R^d \rightarrow \R$ is cost function estimating how well parameter $w$ fits the data of the $i$-th sample. Throughout this work, we assume that each $f_i$  has Lipschitz continuous derivatives. Also, we assume that both each $f_i$ and $P(w)$, are strongly convex.

Many problems in applications are often formulated as Problem \eqref{eq1-1}. For example, when setting $f_i(w)=\frac{1}{2}(x_i^Tw-y_i)^2+\frac{\lambda}{2}\|w\|^2$, where $\lambda$ is a regularization parameter, the Problem \eqref{eq1-1} becomes least squares. However, when setting $f_i(w)=\log (1+\exp [-y_ix_i^Tw])+\frac{\lambda}{2}\|w\|^2$, the Problem \eqref{eq1-1} becomes logistic regression. Some other prevalent models, e.g., SVM \cite{wang2012breaking}, sparse dictionary learning \cite{mairal2009online}, low-rank matrix completion \cite{bhojanapalli2016global} and deep learning \cite{sutskever2013importance},  can be written in the form of (\ref{eq1-1}).

To proceed with the analysis of the proposed algorithm, we require making the following common assumptions.
\begin{assumption}
\label{ass1-1}
Each convex function, $f_{i}(w)$, in (\ref{eq1-1}) is $L$-Lipschitz smooth, i.e., there exists $L>0$ such that for all $w$ and $v$ in $\R^{d}$,
\begin{eqnarray}
\|\nabla f_{i}(w)-\nabla f_{i}(v)\| \leq L\parallel w-v\parallel. \label{eq1-3-22}
\end{eqnarray}
\end{assumption}

 Note that this assumption implies that the objective function, $P(w)$, is also $L$-Lipschitz smooth. Moreover, by the property of $L$-Lipschitz smooth function (see in \cite{Nesterov2004Introductory}), we have
 \begin{eqnarray}
P(w)\leq P(v)+ \nabla P(v)^T(w-v) \ +\frac{L}{2}\parallel w-v\parallel^2. \label{eq1-3-23}
\end{eqnarray}

\begin{assumption}
\label{ass1-2}
$P(w)$ is $\mu$-strongly convex, i.e., there exists $\mu>0$ such that for all $w$, $v \in \mathbb{R}^{d}$,
\begin{eqnarray}
(\nabla P(w)-\nabla P(v))^{T}(w-v) \geq \mu \parallel w-v \parallel^{2}, \label{eq1-6}
\end{eqnarray}
or equivalently
\begin{eqnarray}
P(w)\geq P(v) + \nabla P(v) ^{T}(w-v)+\frac{\mu}{2}\|w-v\|^{2} \label{eq1-6-2}.
\end{eqnarray}
\end{assumption}

When setting $w_{*}= \argmin_{w} P(w)$, it is known in \cite{bottou2018optimization} that the strong convexity of $P(w)$ implies that
\begin{eqnarray}
2\mu [P(w)-P(w_*)] \leq \|\nabla P(w)\|^2, \forall w \in \R^d \label{scp-1}.
\end{eqnarray}

In this paper, the complexity analysis aims to bound the number of iterations (or total number of stochastic gradient evaluations) which requires $\E[\|\nabla F(w)\|^2] \leq  \varepsilon$. In this case, we say that $w$ is an $\varepsilon$-accurate solution.

\section{The Algorithm}
\label{our method}
In the following, we begin with the introduction of the RBB step size, and then we put forward our MB-SARAH-RBB method, which incorporates the RBB step size into MB-SARAH.
\subsection{Random Barzilai-Borwein Step Size}
To solve Problem (\ref{eq1-1}), Yang et al. \cite{yang2018random} proposed to use the RBB method to calculate step size for mS2GD, thereby obtaining: mS2GD-RBB. In the inner loop of mS2GD-RBB, the updating scheme of solution sequence is:
\begin{eqnarray}
w_{k+1}=w_k-\eta_k v_k,  \label{mS2GD-RBB}
\end{eqnarray}
where $\eta_k$ is the step size sequence and defined as:
\begin{eqnarray}
\eta_{k}=\frac{1}{b_H}\frac{\|w_{k}-w_{k-1}\|^{2}}{((w_{k}-w_{k-1})^{T}(\nabla P_{S_H}(w_{k})-\nabla P_{S_H}(w_{k-1})))}, \label{eqRBB-1}
\end{eqnarray}
and $v_k$ is the stochastic estimate of $\nabla P(w)$ and defined as:
\begin{eqnarray}
v_k=\nabla P_{S}(w_k)-\nabla P_{S}({\widetilde{w}})+\nabla P(\widetilde{w}), \label{var-1}
\end{eqnarray}
where $\nabla P_{S_H}(w_{k})= \frac{1}{b_H} \sum_{i \in S_H}\nabla f_{i}(w_{k})$, $\nabla P_{S_H}(w_{k-1})= \frac{1}{b_H} \sum_{i \in S_H}\nabla f_{i}(w_{k-1})$, $S_H\subset
\{1, \ldots, n\}$ with size $b_H$, $\nabla P_{S}(w_{k})= \frac{1}{b} \sum_{i \in S}\nabla f_{i}(w_{k})$, $S \subset
\{1, \ldots, n\}$ with size $b$, and $\widetilde{w}$ is an snapshot vector for which the gradient, $\nabla P(\widetilde{w})$, has already been previously calculated in the deterministic step.

Actually, the RBB method satisfies the so-called quasi-Newton property under the  background of stochastic optimization. Specifically, the RBB method can be viewed as a variant of stochastic quasi-Newton method, where the second order information was used. During recent years, more and more researchers and communities show that stochastic quasi-Newton iterates almost as fast as a first order stochastic gradient but only needs less iterations to achieve the same accuracy \cite{bordes2009sgd,byrd2012sample,byrd2016stochastic,agarwal2017second,tripuraneni2018stochastic}.

%\begin{eqnarray}
%\min \limits_{\omega} P(\omega)=F(\omega)+R(\omega) \label{eq1}
%\end{eqnarray}
%where $F(\omega)$ is the average of many probability functions in the CRFs model, i.e.,¡¡
%\begin{eqnarray}
%F(\omega)=\frac{1}{n} \sum\limits_{i=1}^n f_{i}(\omega)=\frac{1}{n} \sum\limits_{i=1}^n -\log p(y_{i}|x_{i},\omega),\label{eq2}
% \end{eqnarray}
% The generative probability model $p(y_{i}| x_{i},\omega)$, is composed of the exponential function, i.e.,
% \begin{eqnarray}
%p(y|x,\omega)=\frac{exp(\omega^{T} F(y,x))}{\sum \limits_{y^{'}} exp(\omega^{T} F(y^{'},x)} \label{eq3}
%\end{eqnarray}
%These models contain a collection of training examples $(x_{i},y_{i})$, where $x_{i} \in R^{k}$ is a feature vector and $y_{i} \in R$ is the desired structure output. In (\ref{eq3}), $F(x,y)$ is a local feature, containing the state and transfer characters. In (\ref{eq1}), $R(\omega)$ is the regularization parameter. The popularly used regularization term includes $R(\omega)=\lambda_{1}\|\omega\|_{1}$ (the Lasso), $R(\omega)=\frac{\lambda_{2}}{2}\|\omega\|_{2}^{2}$ (ridge regression), $R(x)=\lambda_{1}\|\omega\|_{1}+\frac{\lambda_{2}}{2}\|\omega\|_{2}^{2}$ (elastic net). For simplicity, in the remainder of this paper, these terms are referred to as CRFs-LR, CRFs-RR, and CRFs-EN, respectively. Most literature considers just the CRFs-RR model. In this paper, in addition to the CRFs-RR model, the CRFs-LR model is also considered.

\subsection{The proposed method}
\label{section2.2}
The MB-SARAH method, proposed by Nguyen et al. \cite{nguyen2017stochastic}, is viewed as a variant of mS2GD. However, the pivotal difference between the mS2GD and MB-SARAH is that the latter uses a new kind of stochastic estimate of $\nabla P(w_k)$, i.e.,
 \begin{eqnarray}
v_k=\nabla P_{S}(w_{k})-\nabla P_{S}(w_{k-1})+v_{k-1}  \label{al-eq-1}
 \end{eqnarray}

 For comparison, the stochastic estimate of mS2GD-RBB is written in a similar way as \eqref{var-1}. Note that for mS2GD-RBB, $v_k$ is an unbiased estimator of the gradient, i.e., from (\ref{var-1}), we have $\E[v_k]=\nabla P(w_k)$. However, it's not true for MB-SARAH.

We introduce the RBB method into MB-SARAH and  obtain a new SARAH method referred to as MB-SARAH-RBB. But different with mS2GD-RBB, when computing the random step size in MB-SARAH, we multiply a parameter, $\gamma$, in (\ref{eqRBB-1}), i.e.,
\begin{eqnarray}
\eta^{'}_{k}=\frac{\gamma}{b_H}\frac{\|w_{k}-w_{k-1}\|^{2}}{((w_{k}-w_{k-1})^{T}(\nabla P_{S_H}(w_{k})-\nabla P_{S_H}(w_{k-1})))}, \label{eqRBB-2}
\end{eqnarray}
where the parameter, $\gamma$, is important to control the convergence of MB-SARAH-RBB.

Now we are ready to describe our MB-SARAH-RBB method (Algorithm \ref{alg1}).
\begin{algorithm}[h]
   \caption{MB-SARAH-RBB}
   \label{alg1}
\begin{algorithmic}
   \State \textbf{Parameters:} update frequency $m$, samples sizes $b$ and $b_H$, initial point $\widetilde{w}_{0}$, initial step size $\eta_{0}$, a positive constant $\gamma$.
   %\State \textbf{Initialize} $\widetilde{w}_{0}$
   \State {\bfseries for} $s=1, 2, \ldots, $ {\bfseries do}
   \State $w_0=\widetilde{w}_{s-1}$
\State $v_0=\nabla P(w_0)$
\State $w_1=w_0-\eta_0 v_0$
   \For{$k=1$ {\bfseries to} $m-1$}
\State Randomly pick subset $S \subset \{1, \ldots,n\}$ of size $b$,
\vspace{-10pt}
\State
\begin{eqnarray*}
v_k=\nabla P_{S}(w_{k})-\nabla P_{S}(w_{k-1})+v_{k-1}
\end{eqnarray*}
\State Randomly pick subset $S_H \subset \{1, 2, \ldots,n\}$ of size $b_H$, compute a RBB step size:
\State $\eta^{'}_{k}=\frac{\gamma}{b_{H}}\cdot \|w_{k}-w_{k-1}\|_{2}^{2}/((w_{k}-w_{k-1})^{T}(\nabla P_{S_{H}}(w_{k})-\nabla P_{S_{H}}(w_{k-1})))$
%\vspace{-30pt}
\State $w_{k+1}=w_{k}-\eta_{k} v_k$
   \EndFor
\State $\widetilde{w}_{s}=w_{m}$
   \State {\bfseries end for}
\end{algorithmic}
\end{algorithm}

\textbf{Remark:} At the beginning of MB-SARAH-RBB, a step size, $\eta_0$, requires to be specified. However, we observed from the numerical experiments that the performance of MB-SARAH-RBB is not sensitive to the choice of $\eta_0$. It also can be seen from Algorithm \ref{alg1} that, if we always set $\eta_k=\eta$, then MB-SARAH-RBB is reduced to the original MB-SARAH method.

%As seen in \textbf{Algorithm 1}, we introduce a new iteration sequence, $\{\nu_{k}\}$ ( linear combination of the sequence $\{\omega_{k-1}, \omega_{k}\}$), in the outer layer. Instead of using a simple linear communication between the previous and current information , $\{\omega_{k-1}, \omega_{k}\}$, in the stochastic gradient method with a momentum term, we introduce a new updated sequence $\{t_{k}\}$, where $t_{k}$ satisfies
%\begin{eqnarray}
%t_{k+1}^{2}-t_{k+1}-t_{k}^{2}=0
%\end{eqnarray}
%and $t_{k}>0$.
%
%
%
%
%The pivotal difference between \textbf{Algorithm 1} and \textbf{Algorithm 2} is that we introduce the proximal operator into the inner layer iteration of \textbf{Algorithm 2}. Numerical results show the effectiveness of our proposed method for training the CRFs-LR model.\\

\section{Convergence Analysis}
\label{cona}
In this section, we finish the proof of convergence analysis of MB-SARAH-RBB and discuss its complexity. We first provide the following lemmas.
%\begin{lemma}
%\label{lemma-1}
%If $P(w)$: $\mathbb{R}^{d} \rightarrow \mathbb{R}$ is convex and its gradient is Lipschitz continuous, then for all $w, v \in  \mathbb{R}^{d}$
%\begin{align}
%P(w)&\geq & P(v)+\nabla P(w)^T(w-v) +\frac{1}{2L} \|\nabla P(w)-\nabla P(v)\|_{2}^{2}.  \label{eq1-3-1}
%\end{align}
%\end{lemma}

\begin{lemma}\label{derivation_alg}
Under Assumption \ref{ass1-1}, consider MB-SARAH-RBB within one single outer loop in Algorithm~\ref{alg1}, then we obtain
\begin{align*}
&\sum_{k=0}^{m} \E[ \| \nabla P(w_{k})\|^2 ]  \leq \frac{2\mu b_H}{\gamma}\E[P(w_0)-P(w_*)]\\
& +\sum_{k=0}^{m}  \E [\|\nabla P(w_k)-v_k\|^2]-\left(1-\frac{L\gamma}{\mu b_H}\right) \sum_{k=0}^{m} \E\left[\|v_k\|^2\right].   \label{eq:l1}
\end{align*}
%where $w_{*}$ is a global minimizer of $P$.
\end{lemma}
\begin{proof}
Available in Appendix \ref{appa}
\end{proof}
With minor modification of Lemma 3 in \cite{nguyen2017stochastic}, we obtain the following lemma showing the upper bound for $\mathbb{E}[ \| \nabla P(w_{k}) - v_{k} \|^2 ]$.

\begin{lemma}\label{lemva}
Under Assumption \ref{ass1-1}, consider $v_{k}$ defined by \eqref{al-eq-1} in MB-SARAH-RBB, then for any $k\geq 1$,
\begin{align*}
\E[ \| \nabla P(w_{k}) - v_{k} \|^2 ]  \leq \frac{L^2\gamma^2}{\mu^2 b b_H^2} \left( \frac{n-b}{n-1} \right) \sum_{j=1}^{k} \mathbb{E}[\| v_{j-1} \|^2].
\end{align*}
\end{lemma}

Using the above lemmas, we obtain the following convergence rate for MB-SARAH-RBB with one outer loop.

%Also, we need the following lemmas which are given in \cite{nguyen2017sarah,nguyen2017stochastic}.

\begin{theorem}
\label{th1}
 Under Assumptions \ref{ass1-1}, \ref{ass1-2} and Lemmas \ref{derivation_alg}, \ref{lemva}, let $w_{*}=\argmin_{w} P(w)$ and choose $S, S_H \subset \{1, \ldots, n\}$ with size $b$ and $b_H$ at random, respectively. Consider MB-SARAH-RBB (within one outer loop in Algorithm \ref{alg1}) with
 \begin{eqnarray}
\frac{L^2\gamma^2}{\mu^2 b b_H^2} \left( \frac{n-b}{n-1} \right)m-\left(1-\frac{L\gamma}{\mu b_H} \right)\leq 0, \label{th1-f11}
\end{eqnarray}
then we have
 \begin{eqnarray*}
\E[\|\nabla P(w_m)\|^2] \leq  \frac{2\mu b_H}{\gamma(m+1)} [P(w_0)-P(w_*)] \label{eq22}
\end{eqnarray*}
\end{theorem}
%\emph{Proof}.
\begin{proof}
Available in Appendix \ref{appb}
\end{proof}
This result shows that the inner loop of MB-SARAH-RBB with a single outer loop converges sublinearly. Actually, to obtain
 \begin{eqnarray*}
\frac{2\mu b_H}{\gamma(m+1)} [P(w_0)-P(w_*)]\leq \varepsilon, \label{eq22}
\end{eqnarray*}
it is sufficient to choose $m=O\left(\frac{\mu b_H}{\gamma \varepsilon}\right)$. Hence, the total complexity to require an $\varepsilon$-accurate
solution is $n+2m=O\left(n+\frac{\mu b_H}{\gamma \varepsilon}\right)$. Therefore, the following conclusion for complexity bound is obtained.
\begin{corollary}
\label{cor1}
Under Assumption \ref{ass1-1}, consider MB-SARAH-RBB with one outer loop, then $\|\nabla P(w_k)\|^2$ has sublinear convergence in expectation with a rate of $O(\mu b_H/\gamma m)$, and the total complexity to achieve an $\varepsilon$-accurate solution is $O\left(n+\mu b_H/\gamma \varepsilon\right)$.
\end{corollary}

Since $L\gg \mu$, compared with \textbf{Corollary 1} in \cite{nguyen2017stochastic}, we have that the complexity of our MB-SARAH-RBB method is better than the complexity of MB-SARAH which is $O\left(n+\frac{L^2}{\varepsilon^2}\left(\frac{n-b}{n-1}\right)\right)$ when choosing an appropriate mini-batch size $b_H$. Note that, in our MB-SARAH-RBB method, the parameter, $\gamma$, is greater than $\varepsilon$.

Now, we present the estimating convergence of MB-SARAH-RBB with multiple outer steps.

\begin{theorem}
\label{th2}
Under Assumptions \ref{ass1-1}, \ref{ass1-2} and Lemmas \ref{derivation_alg}, \ref{lemva},  let $w_{*}=\argmin_{w} P(w)$ and set $S, S_H \subset \{1, \ldots, n\}$ with size $b$ and $b_H$ at random, respectively. Consider MB-SARAH-RBB with
 \begin{eqnarray*}
\frac{L^2\gamma^2}{\mu^2 b b_H^2} \left( \frac{n-b}{n-1} \right)m-\left(1-\frac{L\gamma}{\mu b_H} \right)\leq 0, \label{th1-f1}
\end{eqnarray*}
then we have
 \begin{eqnarray*}
\E[\|\nabla P(\widetilde{w}_s)\|^2] \leq  (\rho_m)^s \|\nabla P(\widetilde{w}_0)\|^2 \label{eq23}
\end{eqnarray*}
where $\rho_m=\frac{b_H}{\gamma(m+1)}$.
\end{theorem}

\begin{proof}
Available in Appendix \ref{appc}
\end{proof}

To obtain $\E[\|\nabla P(\widetilde{w}_s)\|^2] \leq  (\rho_m)^s \|\nabla P(\widetilde{w}_0)\|^2 < \varepsilon$, it is sufficient
to set $S=O(\log (1/\varepsilon))$. Therefore, we have the following conclusion for the total complexity of the proposed method.
\begin{corollary}
\label{cor2}
Suppose Assumption \ref{ass1-1} hold, the total complexity of MB-SARAH-RBB to achieve an $\varepsilon$-accurate
solution is $O(\left(n+\mu b_H/\gamma \varepsilon\right)log (1/\varepsilon))$.
\end{corollary}

Compared the complexity of MB-SARAH, Corollary \ref{cor2} indicates that MB-SARAH-RBB has lower complexity when choosing an appropriate mini-batch size $b_H$.

\section{Experiments}
\label{experiments}
In this section, the effectiveness of our MB-SARAH-RBB method is verified with experiments. In particular, our experiments were performed on the well-worn problems of training ridge regression, i.e.,
\begin{eqnarray}
\min \limits_{w \in \mathbb{R}^{d}} ~~P(w):=\frac{1}{n} \sum_{i=1}^n \log(1+\exp(-y_{i}x_{i}^{T}w)) +\frac{\lambda}{2}\|w\|^{2}, \label{experiment-2}
\end{eqnarray}
where $\{(x_i, y_i)\}_{i=1}^{n}\subset \R^d \times \{+1,-1\}^n$ is a collection of training examples.

We tested our MB-SARAH-RBB method on the three publicly available data sets (\emph{a8a},\emph{w8a} and \emph{ijcnn1})\footnote{$a8a$, $w8a$ and $ijcnn1$ can be downloaded  from \url{https://www.csie.ntu.edu.tw/~cjlin/libsvmtools/datasets/}.}.  Detailed information of the data sets are listed in Table \ref{table1}.
\begin{table}[h]
\setlength{\abovecaptionskip}{0pt}
\setlength{\belowcaptionskip}{1pt}
\caption{\small{Data information of experiments}}
\label{table1}
\begin{center}
\begin{tabular}{llllll}
\hline

\hline
\multicolumn{1}{c}{ Dataset}  &\multicolumn{1}{c}{Training size} &\multicolumn{1}{c}{feature}  &\multicolumn{1}{c}{\bf $\lambda$}
\\ \hline

\hline
a8a       &22,696 &123 &$10^{-2}$   \\
w8a    &49,749  &300 &$10^{-2}$  \\
ijcnn1    &49,990  &22 &$10^{-4}$  \\
\hline

\hline
\end{tabular}
\end{center}
\end{table}

\subsection{Properties of MB-SARAH-RBB}
\label{section4.0}
In this subsection, we show the properties of MB-SARAH-RBB conducted using data sets listed in Table \ref{table1}. To clearly show the properties of our MB-SARAH-RBB method, we present the comparison results between MB-SARAH-RBB and MB-SARAH with the best-tuned step size. For ease of analysis, in MB-SARAH-RBB, we take the same batch samples, $b$, as MB-SARAH to obtain the solution sequence, on different data sets.  Therefore, we can easily see the cases of batch samples, $b_H$, for obtaining step size sequence.

In addition, for MB-SARAH-RBB, we chose the parameter, $\gamma$, as 0.1 when the batch samples, $b_H$, is small; otherwise, we take $\gamma=1$, or a slightly large number. Moreover, we set $\eta_0=0.1$ for MB-SARAH-RBB.

 Fig. \ref{cp-our-mb-s-a8a}, \ref{cp-our-mb-s-w8a} and \ref{cp-our-mb-s-ijcnn1} compare MB-SARAH-RBB with MB-SARAH. In all sub-figures, the horizontal axis represents the number of effective passes over the data, where each effective pass evaluates $n$ component gradients. The vertical axis is the sub-optimality, $P(\widetilde{w}_s)-P(w_{*})$, where we obtain $w_{*}$ by performing MB-SARAH with the best-tuned step size. Moreover, the dashed lines represent MB-SARAH with different fixed step sizes and the solid lines represent MB-SARAH-RBB with different batch sizes $b$ and $b_H$. The detailed information of the parameters is given in the legends of the sub-figures.

\begin{figure*}[htbp]
   \centering
    \includegraphics[width=0.40\textwidth]{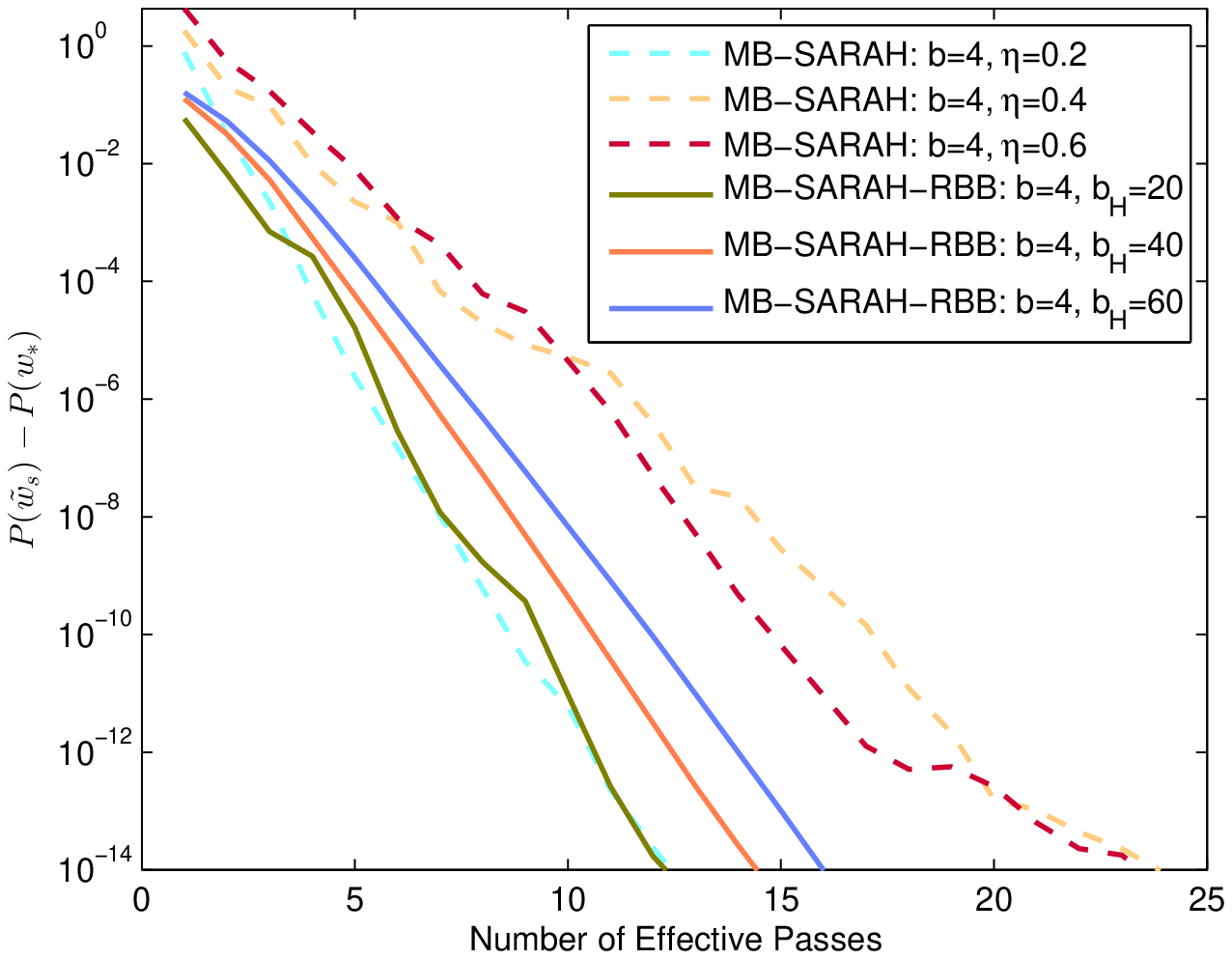}
    \includegraphics[width=0.40\textwidth]{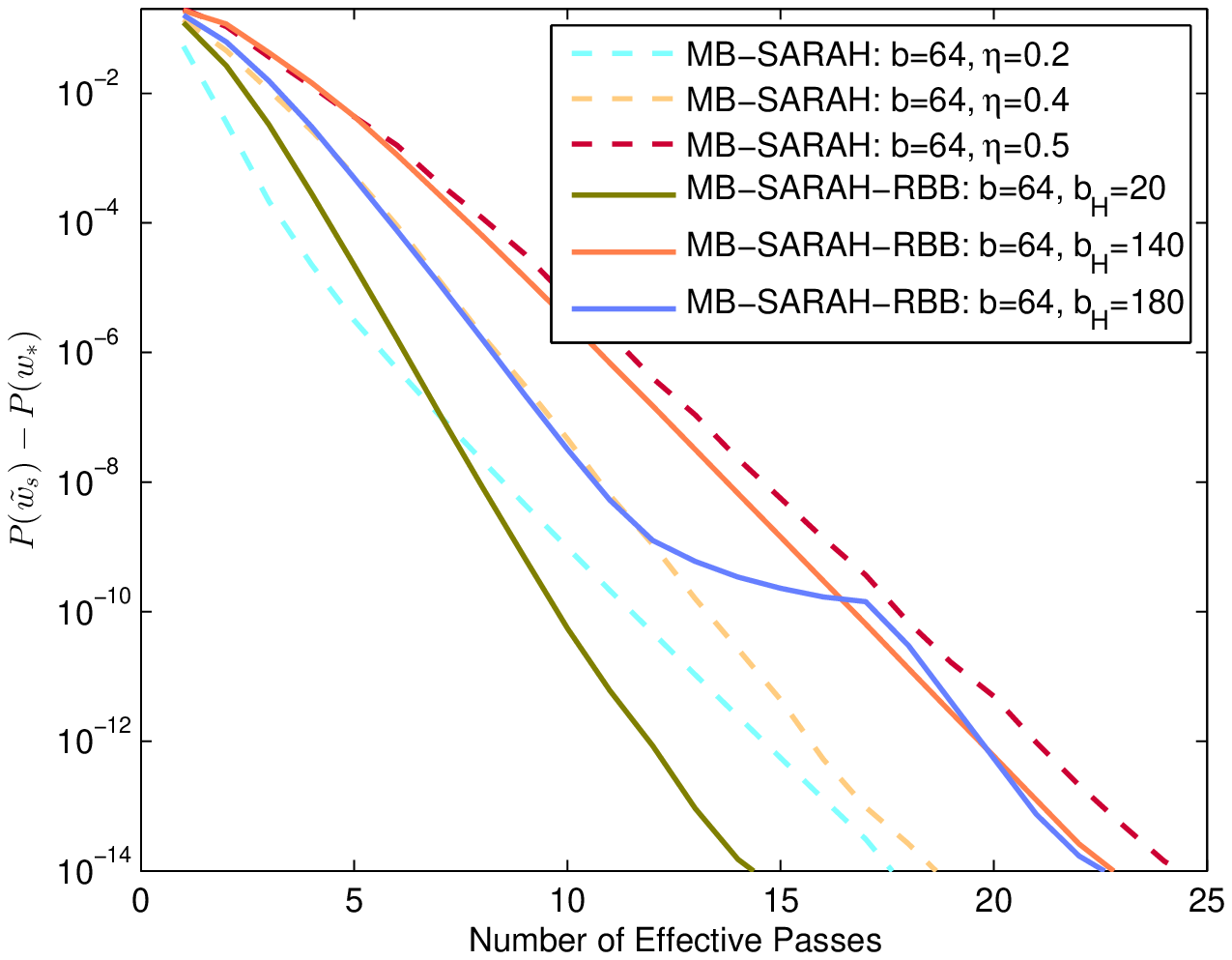}
    \caption{\footnotesize Comparison of MB-SARAH-RBB and MB-SARAH with fixed step sizes on \emph{a8a}. The dashed lines stand for MB-SARAH with different fixed step sizes $\eta$. The solid lines correspond to MB-SARAH-RBB with different mini-batch sizes $b$ and $b_H$.}
  \label{cp-our-mb-s-a8a}
\end{figure*}

\begin{figure*}[htbp]
   \centering
    \includegraphics[width=0.40\textwidth]{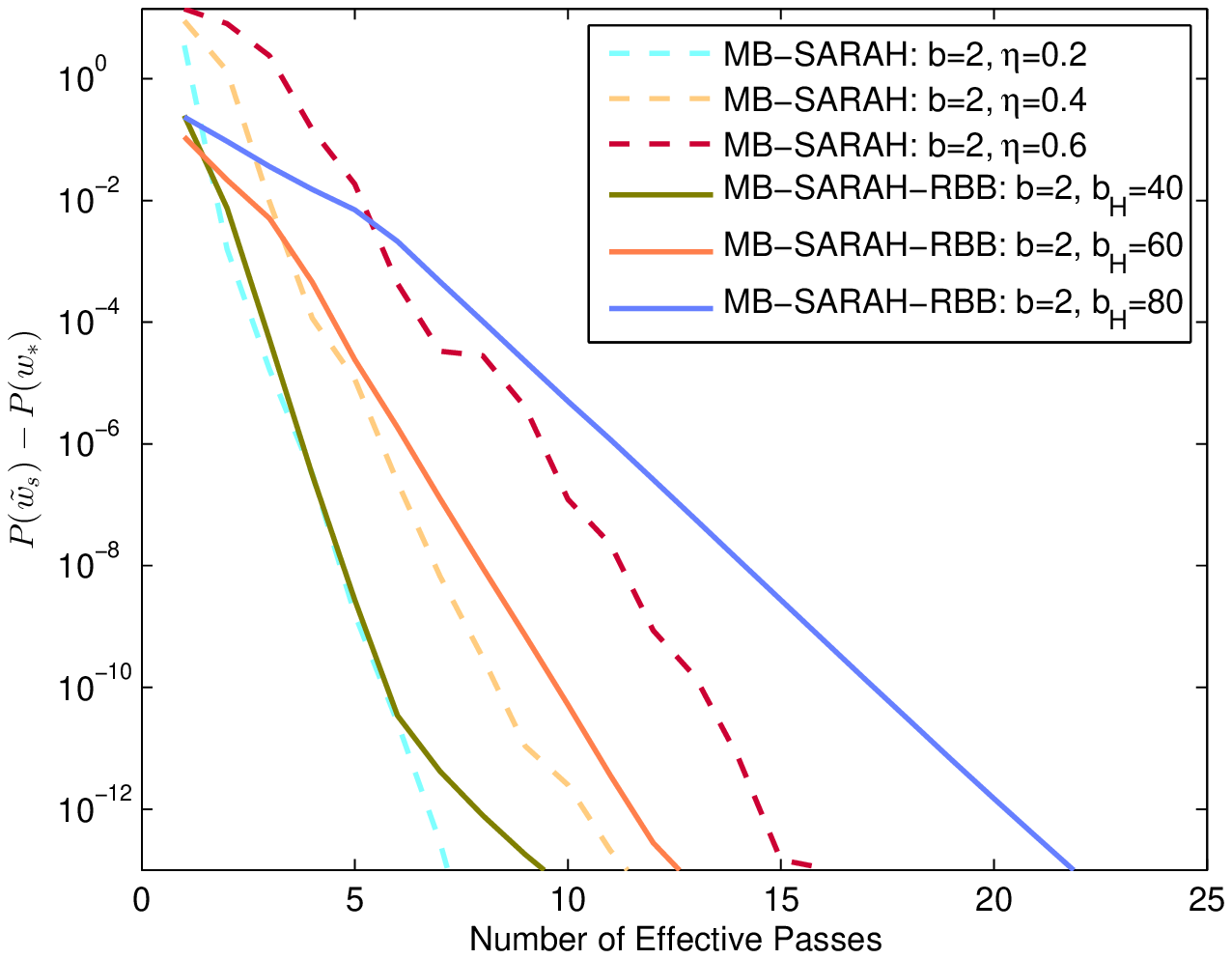}
    \includegraphics[width=0.40\textwidth]{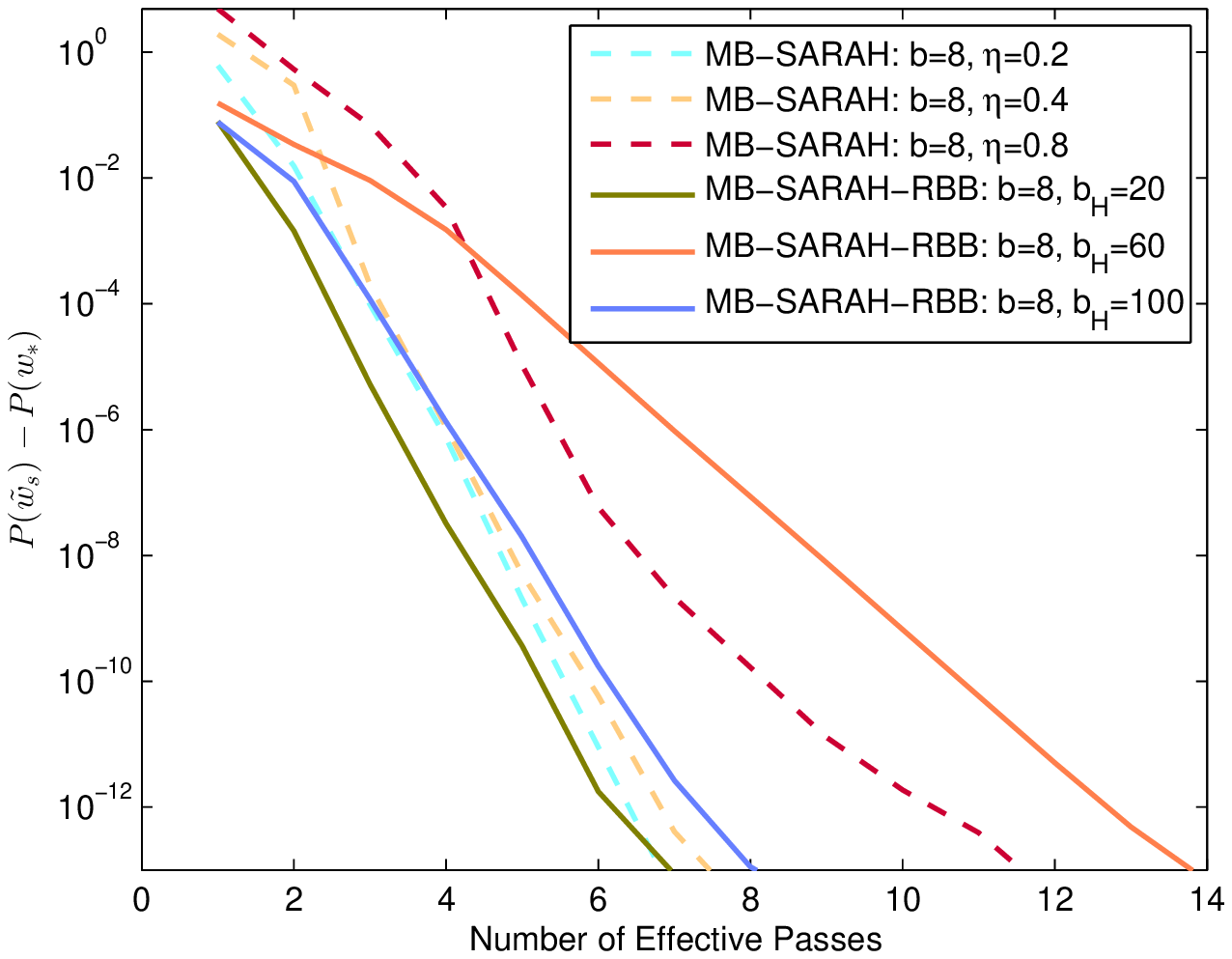}
    \caption{\footnotesize Comparison of MB-SARAH-RBB and MB-SARAH with fixed step sizes on \emph{w8a}. The dashed lines stand for  MB-SARAH with different fixed step sizes $\eta$. The solid lines correspond to MB-SARAH-RBB with different mini-batch sizes $b$ and $b_H$. }
  \label{cp-our-mb-s-w8a}
\end{figure*}

\begin{figure*}[htbp]
   \centering
    \includegraphics[width=0.40\textwidth]{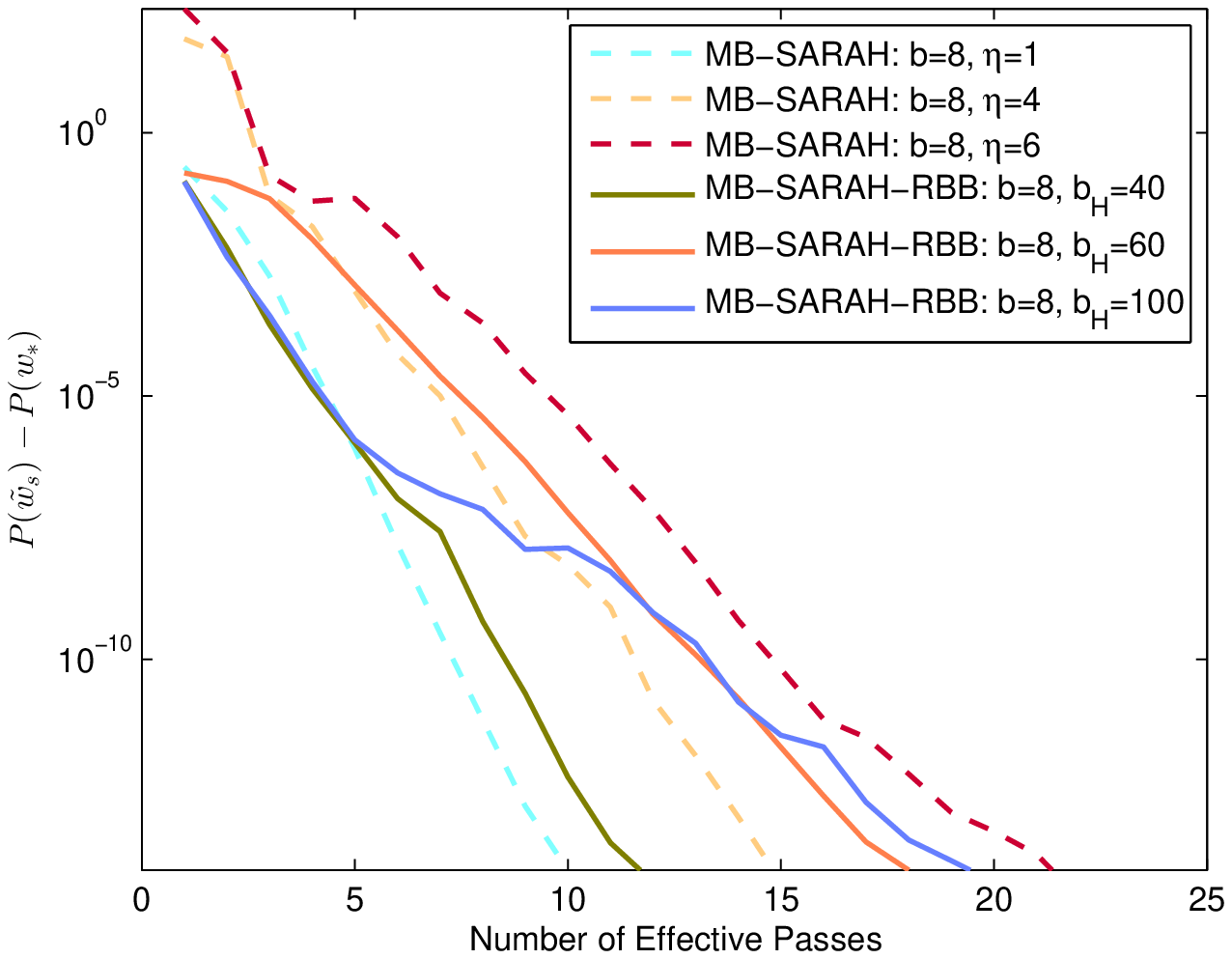}
    \includegraphics[width=0.40\textwidth]{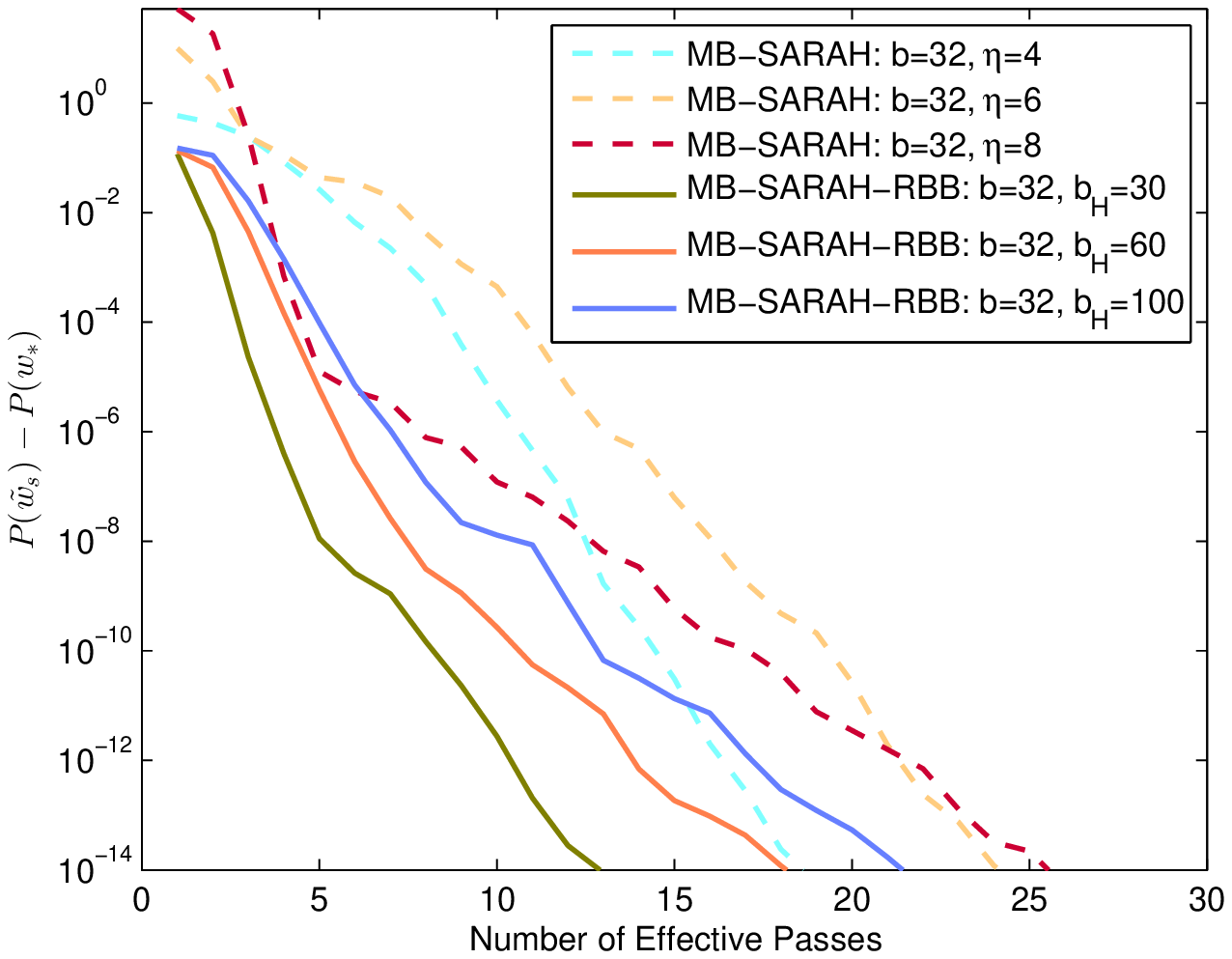}
    \caption{\footnotesize Comparison of MB-SARAH-RBB and MB-SARAH with fixed step sizes on \emph{ijcnn1}. The dashed lines stand for MB-SARAH with different fixed step sizes $\eta$. The solid lines correspond to MB-SARAH-RBB with different mini-batch sizes $b$ and $b_H$.}
  \label{cp-our-mb-s-ijcnn1}
\end{figure*}

 Fig. \ref{cp-our-mb-s-a8a}, \ref{cp-our-mb-s-w8a} and \ref{cp-our-mb-s-ijcnn1} show that, MB-SARAH-RBB is comparable to or performs better than MB-SARAH with best-tuned step size. Also, Fig. \ref{cp-our-mb-s-a8a}, \ref{cp-our-mb-s-w8a} and \ref{cp-our-mb-s-ijcnn1} indicate that, when fixed batch samples, $b$, it is no need to set a large batch samples, $b_H$, to obtain step size sequence. However, a small batch size, $b_H$, makes MB-SARAH-RBB diverge.

 In Algorithm \ref{alg1}, we pointed out that MB-SARAH-RBB is not sensitive to the choice of step size, $\eta_0$. To present this case, we set three different step sizes ($\eta_0=0.01, 0.1, 1$) for MB-SARAH-RBB on $a8a$ and $ijcnn1$ and the results were presented in Fig. \ref{dif-step-size}. We set $b=4, b_H=40$ for $a8a$ and $b=8, b_H=60$ for $ijcnn1$.

 \begin{figure*}[htbp]
   \centering
    \includegraphics[width=0.40\textwidth]{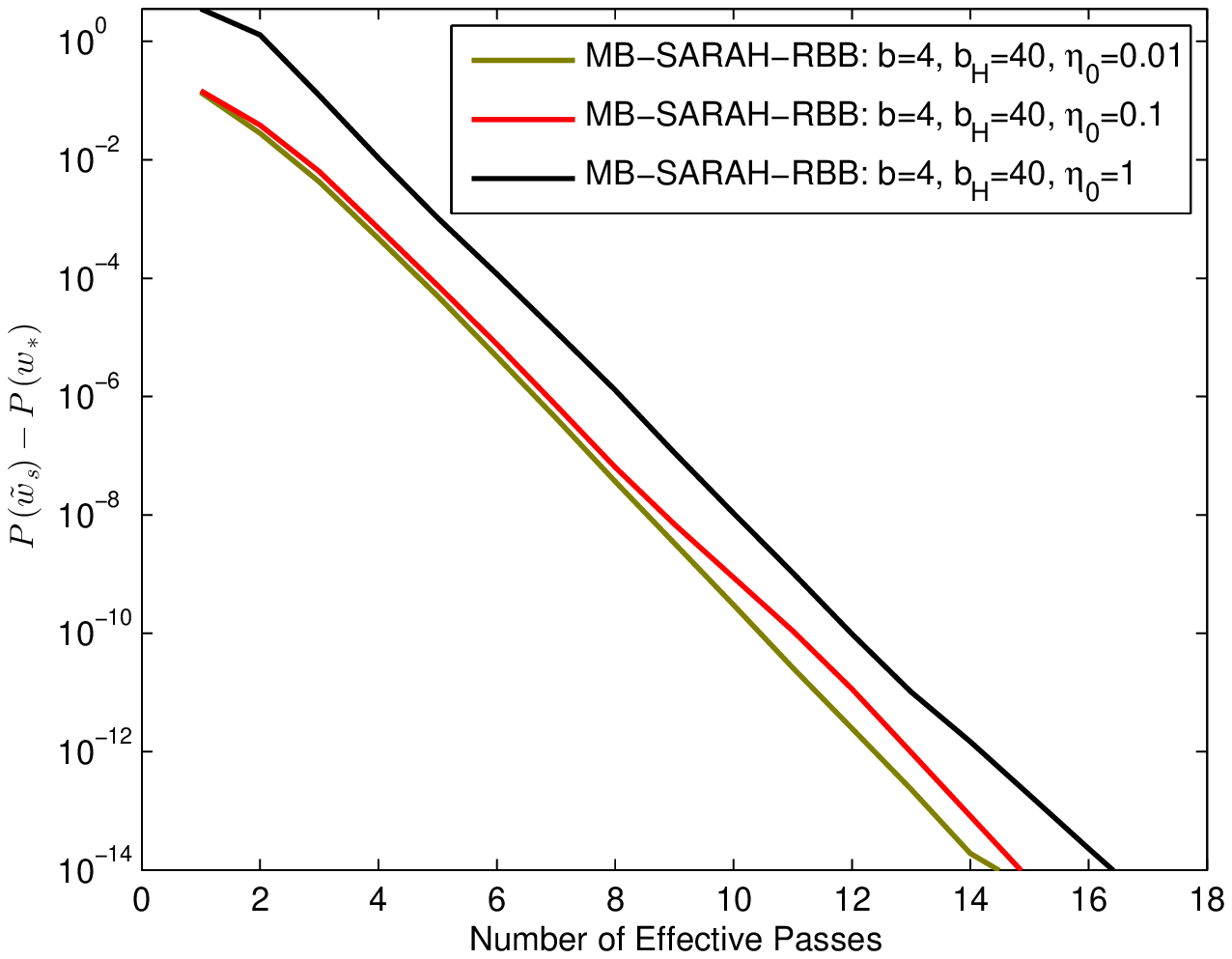}
    \includegraphics[width=0.40\textwidth]{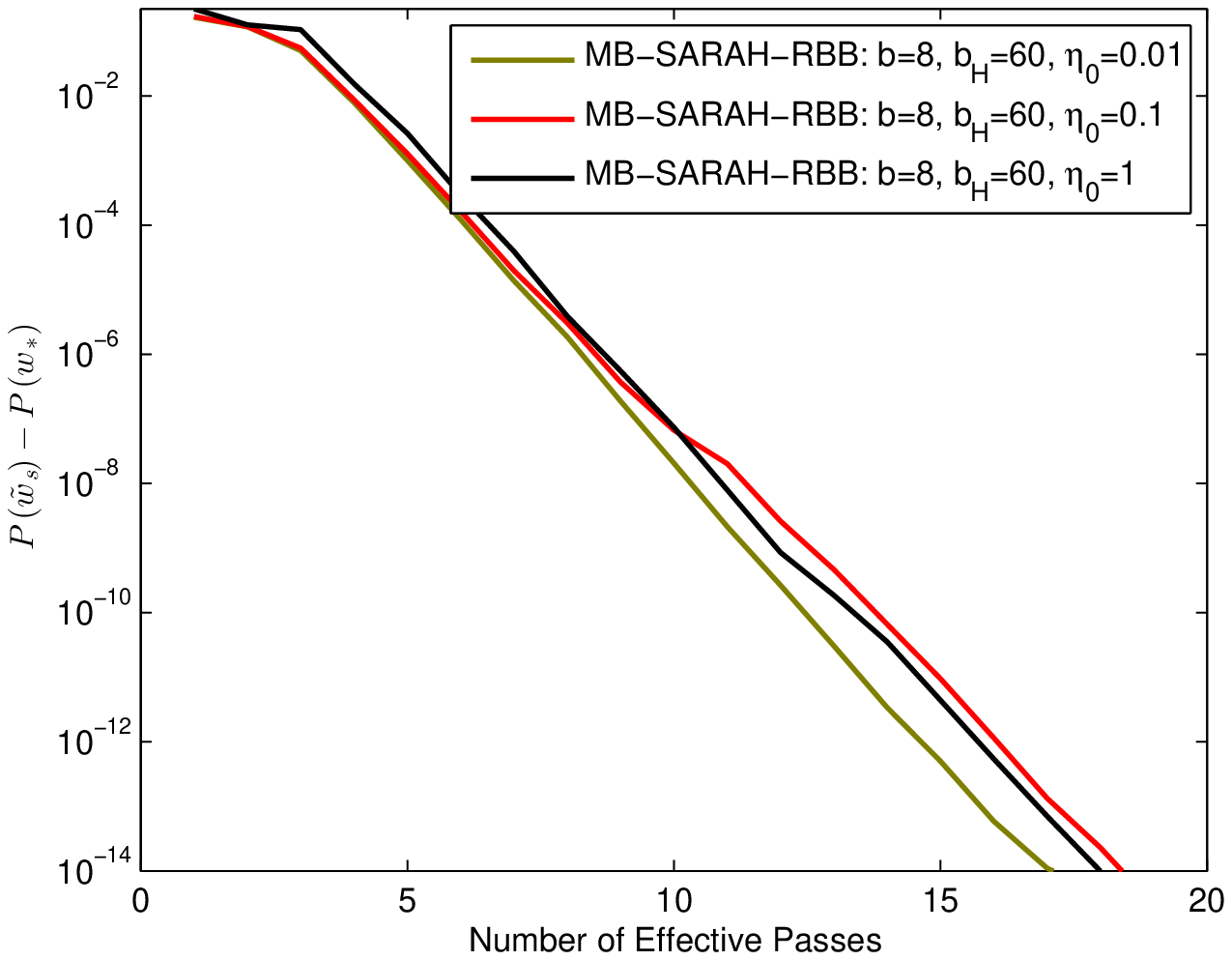}
    \caption{\footnotesize Different initial step sizes for MB-SARAH-RBB on $a8a$ (left) and $ijcnn1$ (right).}
  \label{dif-step-size}
\end{figure*}

 It can be seen from Fig. \ref{dif-step-size} that, the performance of MB-SARAH-RBB is not influenced by the choice of $\eta_0$.

\subsection{Comparison with mS2GD-RBB}
 mS2GD-RBB, proposed by Yang et al. \cite{yang2018random}, uses the similar strategy as us to compute step size for mS2GD. One of the key difference between mS2GD-RBB and MB-SARAH-RBB is that the latter multiply a positive constant, $\gamma$, in \eqref{eqRBB-1}. To further show the efficacy of our MB-SARAH-RBB method, we compare these two methods. All parameters of mS2GD-RBB are set as suggested in \cite{yang2018random}. Also, we use the dashed lines to represent mS2GD-RBB and the solid lines to represent MB-SARAH-RBB.

\begin{figure*}[htbp]
   \centering
    \includegraphics[width=0.40\textwidth]{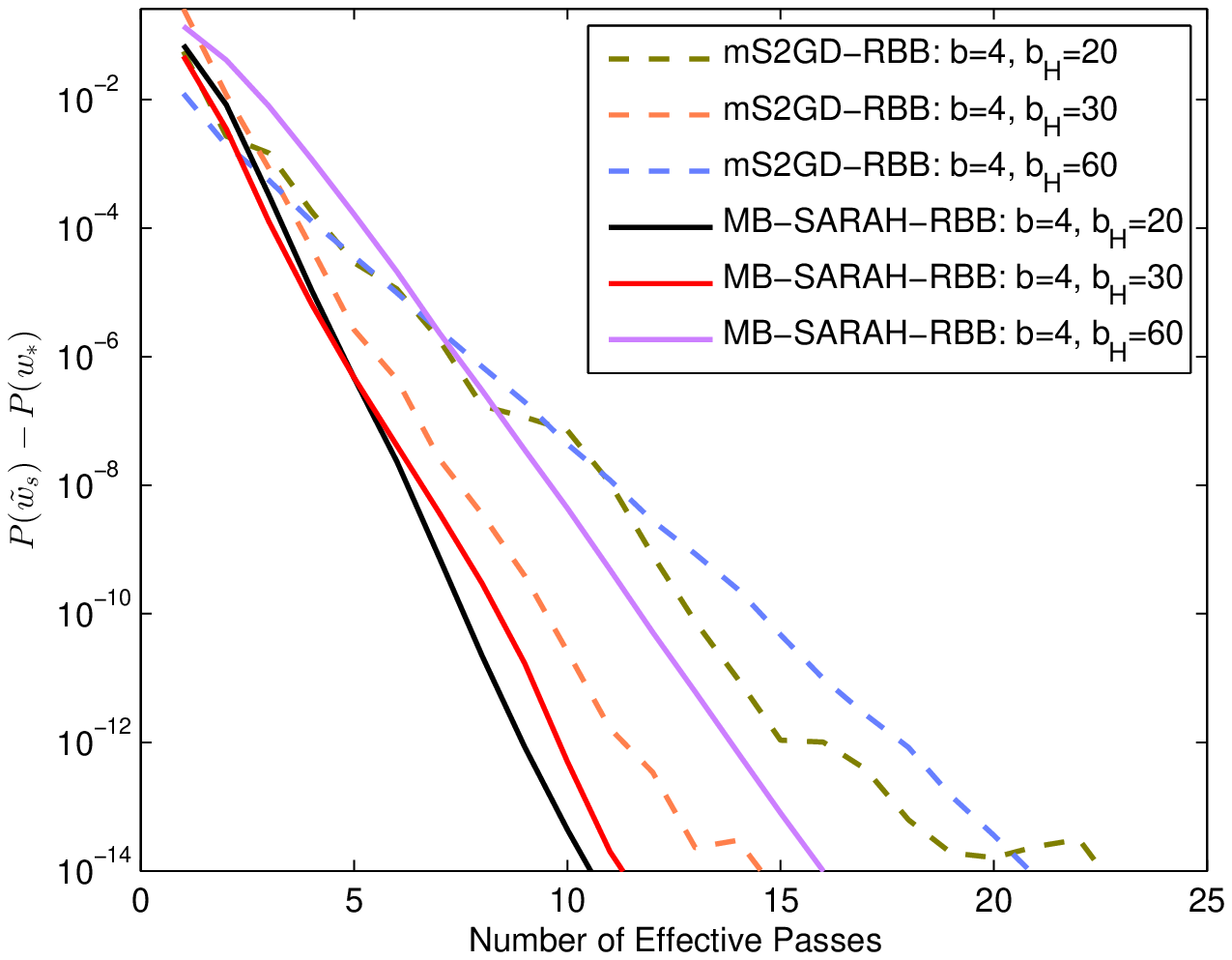}
    \includegraphics[width=0.40\textwidth]{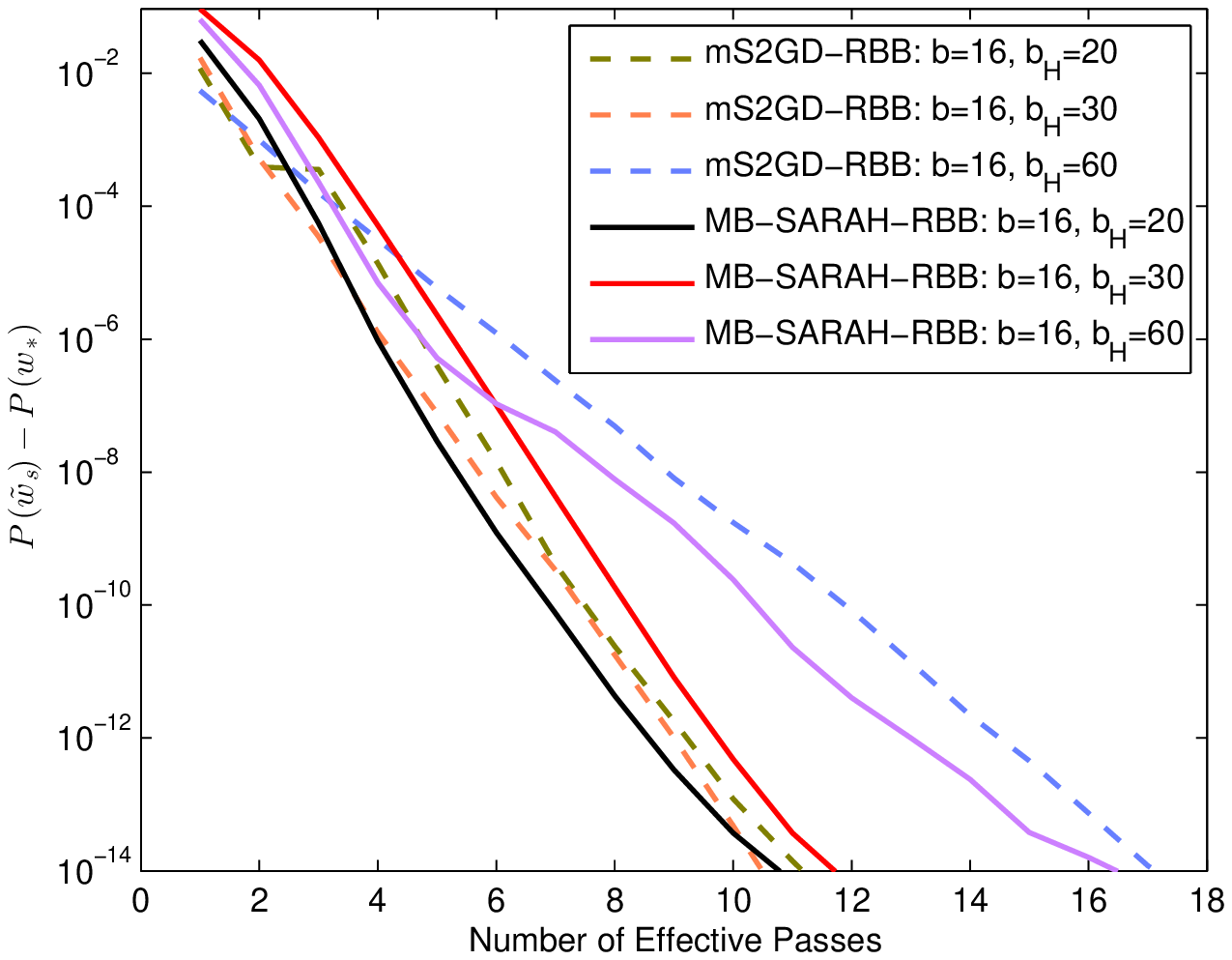}
    \caption{\footnotesize Comparison of MB-SARAH-RBB and mS2GD-RBB on \emph{a8a}. The dashed lines stand for  mS2GD-RBB with different mini-batch sizes $b$ and $b_H$. The solid lines correspond to MB-SARAH-RBB with different mini-batch sizes $b$ and $b_H$.}
  \label{cp-our-ms2gd-a8a}
\end{figure*}

\begin{figure*}[htbp]
   \centering
    \includegraphics[width=0.40\textwidth]{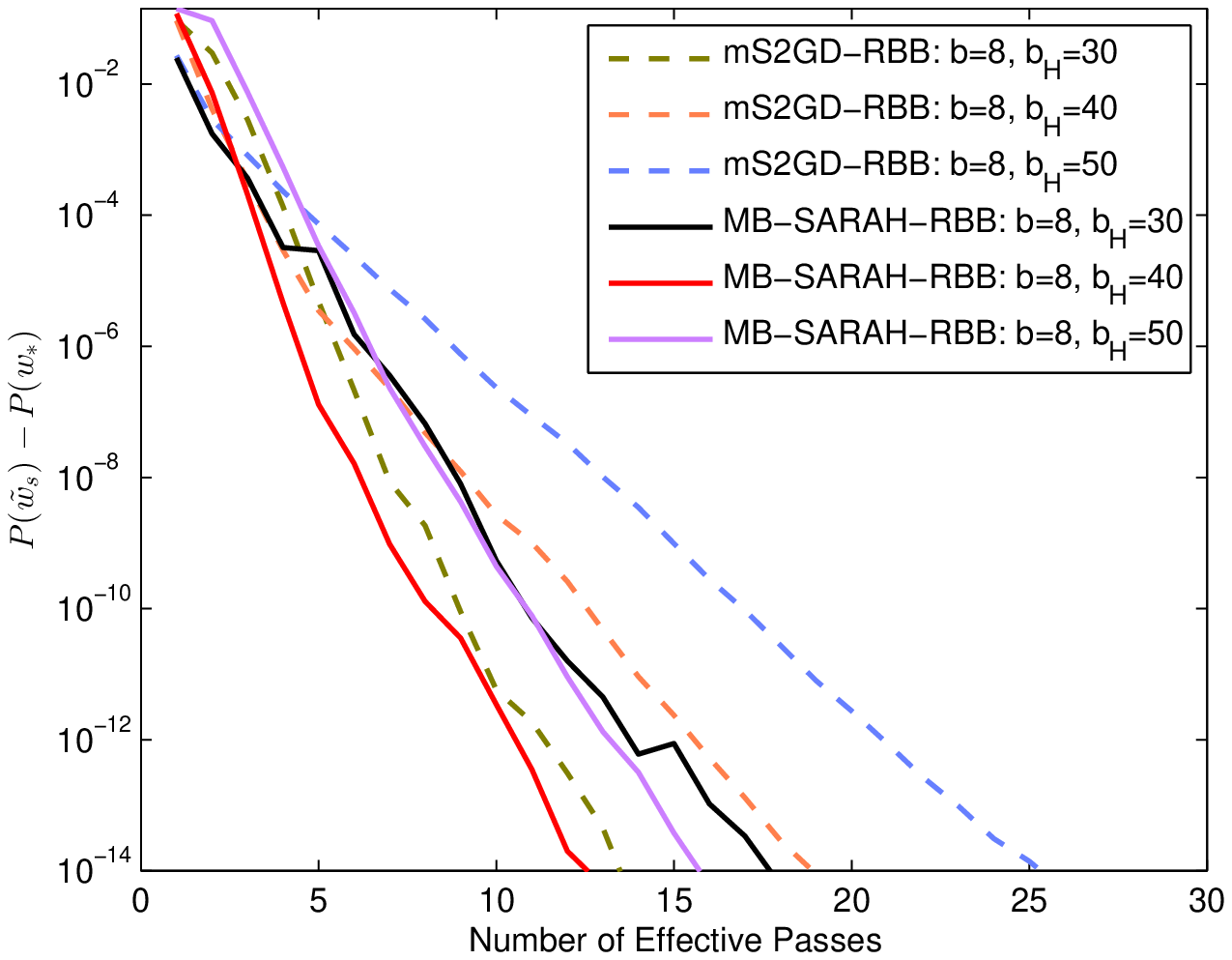}
    \includegraphics[width=0.40\textwidth]{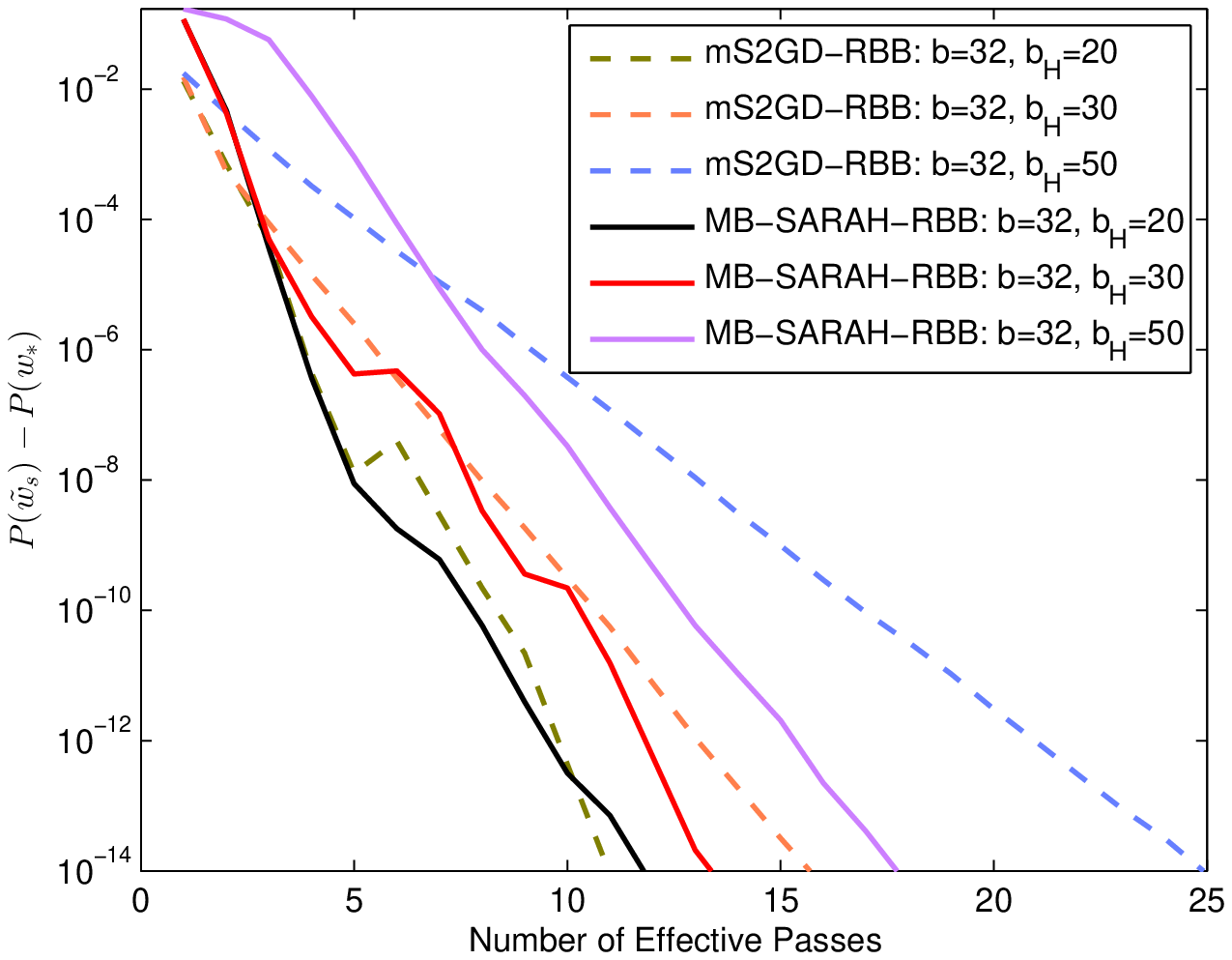}
    \caption{\footnotesize Comparison of MB-SARAH-RBB and  mS2GD-RBB on \emph{ijcnn1}. The dashed lines stand for  mS2GD-RBB with different mini-batch sizes $b$ and $b_H$. The solid lines  correspond to MB-SARAH-RBB with different mini-batch sizes $b$ and $b_H$. }
  \label{cp-our-ms2gd-ijcnn1}
\end{figure*}

Fig. \ref{cp-our-ms2gd-a8a} and \ref{cp-our-ms2gd-ijcnn1} show that our MB-SARAH-RBB method performs better than or is comparable to mS2GD-RBB. It also indicates that the performance of the original MB-SARAH method can be improved by introducing the improved RBB method.

\subsection{Comparison with other related methods}
\label{section4.1}
In this section, we compare our MB-SARAH-RBB method with the following methods:\\
1) \textbf{SAG-LS}: Stochastic average gradient method with line search \cite{schmidt2015non}.\\
2) \textbf{SAG-BB}: Stochastic average gradient method with BB step size \cite{tan2016barzilai}.\\
3) \textbf{SVRG}: Stochastic variance reduction gradient method \cite{johnson2013accelerating}. For SVRG, the best constant step size was employed.\\
4) \textbf{SVRG-BB}: Stochastic variance reduction gradient method with BB step size \cite{tan2016barzilai}.\\
5) \textbf{mS2GD-BB}: A batch version of SVRG-BB proposed in \cite{yang2018mini}. For mS2GD-BB, all parameters were set as suggested in \cite{yang2018mini}.\\
6) \textbf{SDCA}: Stochastic descent coordinate ascent method \cite{shalev2013stochastic}. We chose the parameters as suggested in \cite{shalev2013stochastic}. Also, the best constant step size was employed\\
7) \textbf{Acc-Prox-SVRG}: an version of accelerated stochastic gradient method in \cite{nitanda2014stochastic}. We chose $\eta=1$, $m=\delta b$ ($\delta=10$), and $\beta_{k}=\frac{b-2}{b+2}$ ($b=100$), as suggested in \cite{nitanda2014stochastic}. Also, the best constant step size was employed. \\
8) \textbf{Acc-Prox-SVRG-BB}: an variant of Acc-Prox-SVRG, with the BB step size in \cite{yang2019accelerated}. We set the parameters of Acc-Prox-SVRG-BB as suggested in  \cite{yang2019accelerated}.\\
9) \textbf{Acc-Prox-SVRG-RBB}: an variant of Acc-Prox-SVRG, with the RBB step size in \cite{yang2018random}. For Acc-Prox-SVRG-BB, we set the best batch size $b$ and $b_H$ for different data sets.\\
10) \textbf{MSVRG-OSS}: the MSVRG method with an online step size \cite{yang2019mini} The parameters were set as suggested in \cite{yang2019mini}.\\

 \begin{figure*}[htbp]
 \centering
  \includegraphics[width=5.5cm,height=4.5cm]{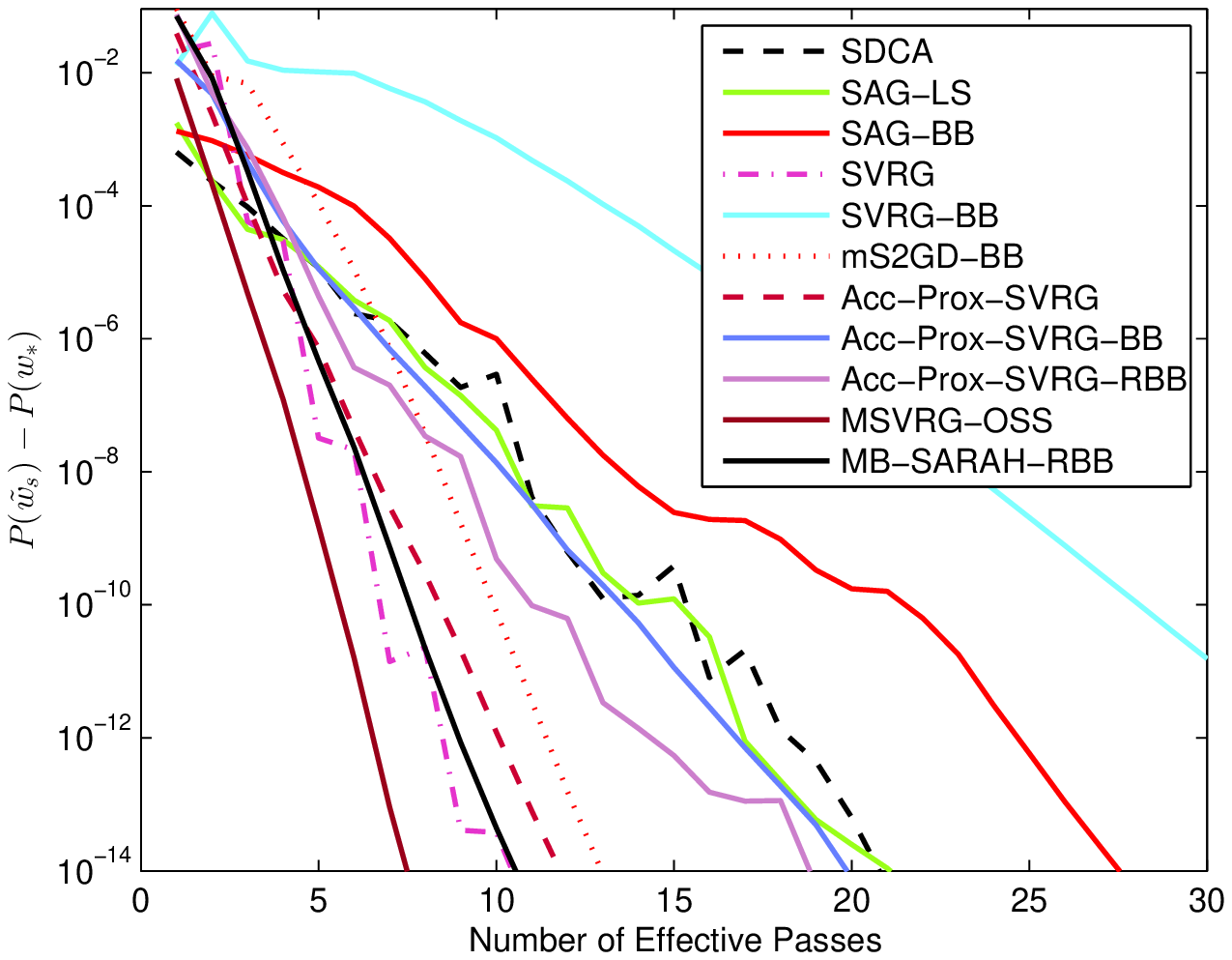}
  \includegraphics[width=5.5cm,height=4.5cm]{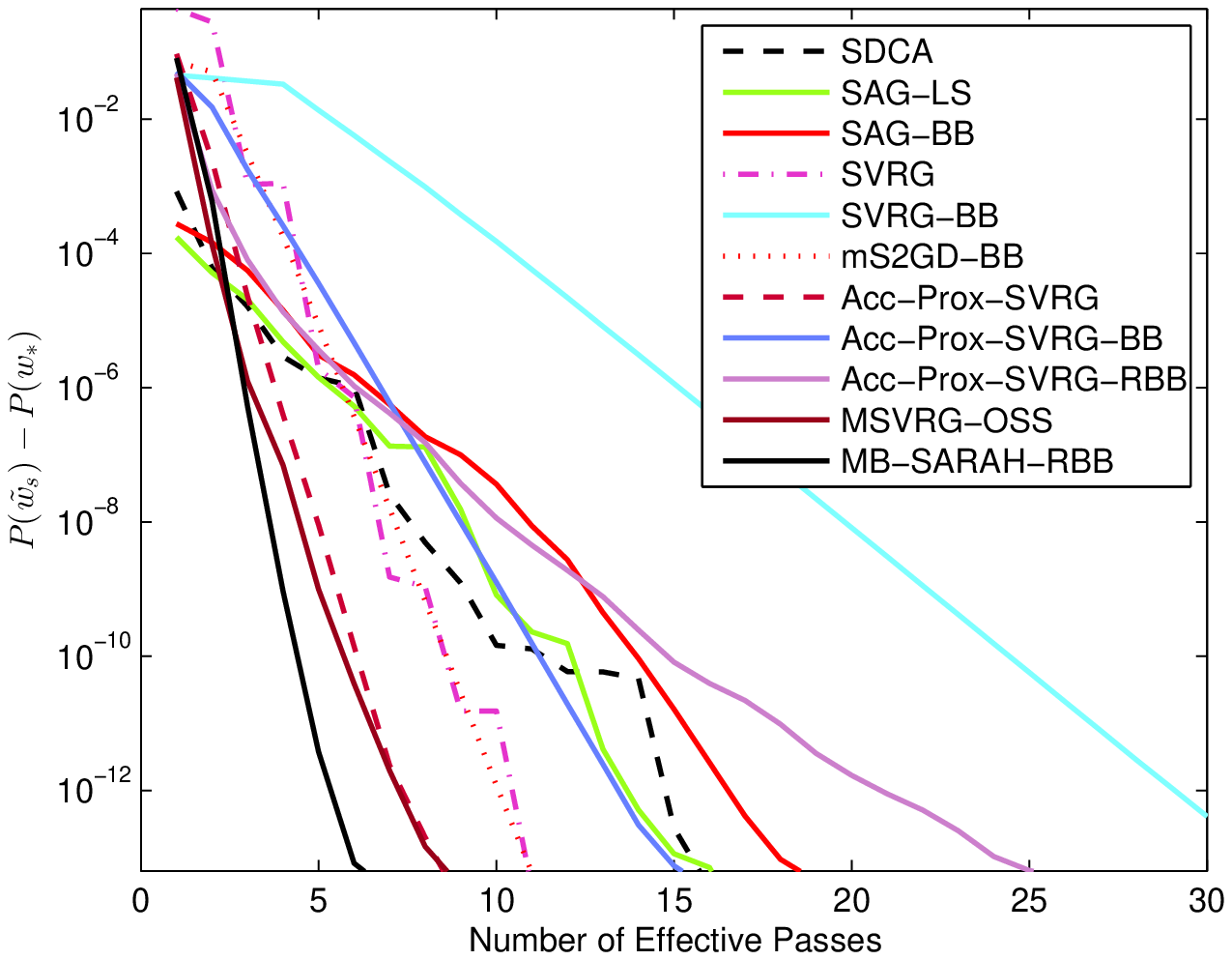}
  \includegraphics[width=5.5cm,height=4.5cm]{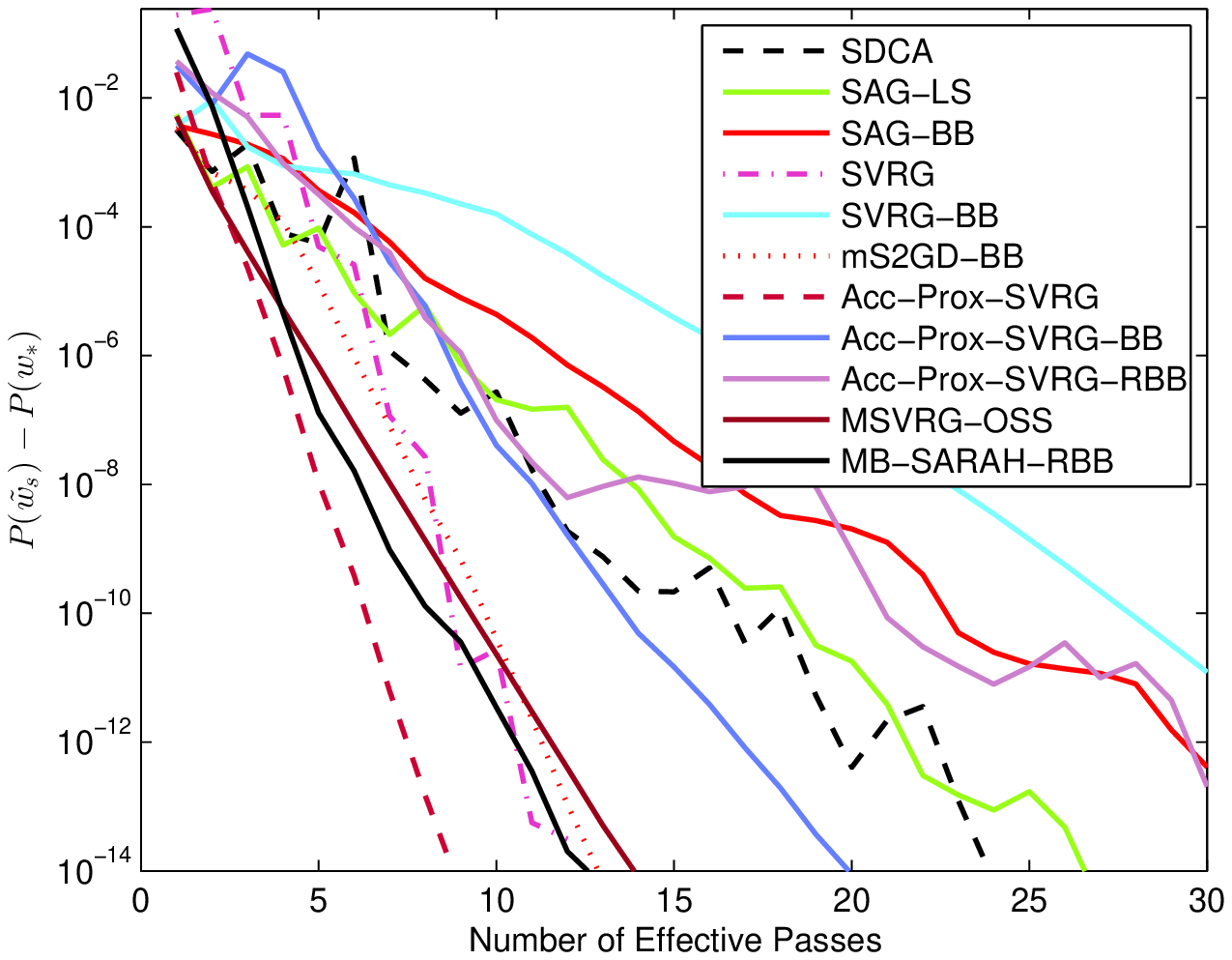}
 \caption{\footnotesize Comparison of different methods on three data sets: $a8a$ (left), $w8a$ (middle) and $ijcnn1$ (right).}
 \label{cpo-t}
 \end{figure*}

As can be seen from Fig. \ref{cpo-t}, our MB-SARAH-RBB method outperforms or matches state-of-the-art algorithms.

\section{Conclusion}
\label{conclusion}
This paper is motivated by a defect related to SARAH for step size choice. Specifically, common implementations of such schemes provide little guidance in specifying step size parameters that prove crucial in practical performance. Accordingly, we propose using the RBB method to automatically evaluate step size for MB-SARAH and obtain MB-SARAH-RBB. We prove that our MB-SARAH-RBB method converges with a linear convergence rate for strongly convex objective functions. We analyze the complexity of MB-SARAH-RBB and show that the complexity of the original MB-SARAH method is improved by combining the RBB method. Numerical results show that our MB-SARAH-RBB method outperforms or matches state-of-the-art algorithms.

%\hfill mds
%
%\hfill August 26, 2015

%\section*{Acknowledgment}
%
%
%This work was supported by the National Natural Science Foundation of China (Grant No. 41471379, No. 61371144) and Fujian Collaborative Innovation Center for Big Data Applications in Governments.

\appendices

\section{Proofs}
\label{proof}
%*******************
% Section
%*******************
%\section{Technical Results}
%Begin we prove lemma and theorem, we first give the following two lemmas, which are presented in \cite{nguyen2017sarah,nguyen2017stochastic}.

%\begin{lemma}\label{lem:var_diff_mb}
%Suppose that Assumption \ref{ass1-1} holds. Consider $v_{t}$ defined by \eqref{al-eq-1} in MB-SARAH-RBB, then for any $k\geq 1$,
%\begin{align*}
%\E[ \| \nabla P(w_{k}) - v_{k} \|^2 ]
%&= \sum_{j = 1}^{k} [ \| v_{j} - v_{j-1} \|^2 ]
% - \sum_{j = 1}^{k} \E[ \| \nabla P(w_{j})
% \\
%& - \nabla P(w_{j-1}) \|^2 ].
%\end{align*}
%\end{lemma}

\subsection{Proof of Lemma \ref{derivation_alg}}
\label{appa}

According to \eqref{eq1-3-23} and $w_{k+1}=w_k-\eta_k v_k $, we have
\begin{align*}
\E[ P(w_{k+1})] & \overset{\eqref{eq1-3-23}}{\leq}  \E[ P(w_{k})] - \eta_k \E[\nabla P(w_{k})^\top v_{k}]\\
& + \frac{L\eta_k^2}{2} \E [ \| v_{k} \|^2 ].
\end{align*}

Employing the strong convexity of $f_i(w)$, we have the following upper boundary for the RBB step size from Algorithm 1.

\begin{eqnarray*}
\label{eta-1}
 \eta^{'}_{k} &=& \frac{\gamma}{b_H}\cdot \frac{\|w_{k}-w_{k-1}\|^{2}}{(w_{k}-w_{k-1})^{T}(\nabla P_{S_H}(w_{k})-\nabla P_{S_H}(w_{k-1}))} \notag\\
 &\leq & \frac{\gamma}{b_H}\cdot \frac{\|w_{k}-w_{k-1}\|^{2}}{\mu \|w_{k}-w_{k-1}\|^{2}}=\frac{\gamma}{\mu b_H}
\end{eqnarray*}

Therefore, we ascertain that
\begin{align*}
&\E[ P(w_{k+1})]\\
 & \leq  \E[ P(w_{k})] - \frac{\gamma}{\mu b_H} \E\left[\nabla P(w_{k})^\top v_{k}\right]
  +  \frac{L\gamma^2}{2\mu^2 b_H^2} \E [ \| v_{k} \|^2 ]\\
&= \E[P(w_k)]-\frac{\gamma}{2\mu b_H}\E \left[\|\nabla P(w_k)\|^2\right]+\frac{\gamma}{2\mu b_H}\E \large[\|\nabla P(w_k)\\
&-v_k\|^2\large] - \left( \frac{\gamma}{2\mu b_H}-\frac{L\gamma^2}{2\mu^2 b_H^2}\right)\E \left[\|v_k\|^2\right],
\end{align*}
 where the last equality is according to the fact that $a^Tb=\frac{1}{2}\left[\|a\|^2+\|b\|^2-\|a-b\|^2\right]$.

By summing over $k=0,\ldots,m$, we have
\begin{align*}
& \E[ P(w_{m+1})]\\
 & \leq \E [P(w_0)]- \frac{\gamma}{2\mu b_H} \sum_{k=0}^{m} \E \left[ \|\nabla P(w_k)\|^2\right]
 + \frac{\gamma}{2\mu b_H}\\
& \cdot \sum_{k=0}^{m} \E [\|\nabla P(w_k)
 -v_k\|^2]-\left(\frac{\gamma}{2\mu b_H}-\frac{L\gamma^2}{2\mu^2b_H^2}\right) \sum_{k=0}^{m} \E\left[\|v_k\|^2\right].
\end{align*}

Further, we have
\begin{align*}
& \sum_{k=0}^{m} \E \left[ \|\nabla P(w_k)\|^2\right] \\
& \leq \frac{2\mu b_H}{\gamma} \E[P(w_0)-P(w_{m+1})] +\sum_{k=0}^{m}  \E \left[\|\nabla P(w_k)-v_k\|^2\right]\\
&-\left(1-\frac{L\gamma}{\mu b_H}\right) \sum_{k=0}^{m} \E\left[\|v_k\|^2\right]\\
& \leq \frac{2\mu b_H}{\gamma}\E[P(w_0)-P(w_*)] +\sum_{k=0}^{m}  \E \left[\|\nabla P(w_k)-v_k\|^2\right]\\
&-\left(1-\frac{L\gamma}{\mu b_H}\right) \sum_{k=0}^{m} \E\left[\|v_k\|^2\right],
\end{align*}
where the last inequality follows since $w_{*}= \argmin_{w} P(w)$.

\subsection{Proof of Theorem \ref{th1}}
\label{appb}
From Lemma \ref{lemva}, we have
\begin{align*}
\E[ \| \nabla P(w_{k}) - v_{k} \|^2 ]  \leq \frac{L^2\gamma^2}{\mu^2 b b_H^2} \left( \frac{n-b}{n-1} \right) \sum_{j=1}^{k} \mathbb{E}[\| v_{j-1} \|^2].
\end{align*}
Since $\|\nabla P(w_0)-v_0\|^2=0$, hence by summing over $k=0,\ldots,m$, we obtain
\begin{align*}
&\sum_{k=0}^{m} \E[ \| \nabla P(w_{k}) - v_{k} \|^2 ]  \leq \frac{L^2\gamma^2}{\mu^2 b b_H^2} \left( \frac{n-b}{n-1} \right) \bigl[m \E[\|v_0\|^2]\\
&+(m-1)\E[ \| v_1 \|^2]+\ldots+\E [\|v_{m-1}\|^2]\bigr].
\end{align*}
Further, we have
\begin{eqnarray}
\sum_{k=0}^{m} \E[ \| \nabla P(w_{k}) - v_{k} \|^2 ]-\biggl(1-\frac{L\gamma}{\mu b_H}\biggr) \sum_{k=0}^{m} \E[\|v_k\|^2]\notag\\
\leq  \frac{L^2\gamma^2}{\mu^2 b b_H^2} \left( \frac{n-b}{n-1} \right) \bigl[m \E[\|v_0\|^2]+(m-1)\E[ \| v_1 \|^2]\notag\\
+\ldots+\E [\|v_{m-1}\|^2]\bigr]-\biggl(1-\frac{L\gamma}{\mu b_H}\biggr) \sum_{k=0}^{m} \E[\|v_k\|^2]\notag\\
 \leq \biggl[ \frac{L^2\gamma^2}{\mu^2 b b_H^2} \left( \frac{n-b}{n-1} \right)m-\left(1-\frac{L\gamma}{\mu b_H}\right)\biggr]\sum_{k=1}^{m}\E[\|v_{k-1}\|^2]\notag\\
\overset{\eqref{th1-f11}}{\leq} 0 ~~~~~~~~~~~~~~~~~~~~~~~~~~~~~~~~~~~~~~~~~~~~~~~~~~~~~~~~~~~~~  \label{eq:apl1}
\end{eqnarray}
Therefore, by Lemma \ref{derivation_alg}, we have
\begin{eqnarray*}
\sum_{k=0}^{m} \E[ \| \nabla P(w_{k})\|^2 ]  \leq \frac{2\mu b_H}{\gamma}\E[P(w_0)-P(w_*)]\\
 +\sum_{k=0}^{m}  \E [\|\nabla P(w_k)-v_k\|^2]-\biggl(1-\frac{L\gamma}{\mu b_H}\biggr) \sum_{k=0}^{m} \E\left[\|v_k\|^2\right]\\
 \overset{\eqref{eq:apl1}}{\leq} \frac{2\mu b_H}{\gamma}\E[P(w_0)-P(w_*)]. ~~~~~~~~~~~~~~~~~~~~~~~~~~~~ \label{eq:l1}
\end{eqnarray*}

By the definition of $\widetilde{w}_s$ in Algorithm \ref{alg1} and $\widetilde{w}_s=w_m$, we have that
\begin{eqnarray*}
\E[\|\nabla P(w_m)\|^2]&=&\frac{1}{m+1}\sum_{k=0}^{m}\E[\|\nabla P(w_k)\|^2]\\
&\leq & \frac{2\mu b_H}{\gamma(m+1)} \E[P(w_0)-P(w_*)]
\end{eqnarray*}

\subsection{Proof of Theorem \ref{th2}}
\label{appc}

Note that $w_0=\widetilde{w}_{s-1}$ and $\widetilde{w}_s=w_m$, $s\geq 1$. From Theorem \ref{th1}, we obtain
\begin{eqnarray*}
\E[\|\nabla P(\widetilde{w}_s)| \widetilde{w}_{s-1}\|^2]&=&\E[\|\nabla P(\widetilde{w}_s)|w_0\|^2]\\
&\leq & \frac{2\mu b_H}{\gamma(m+1)} \E[P(w_0)-P(w_*)] \\
&\overset{\eqref{scp-1}}{\leq}& \frac{b_H}{\gamma(m+1)}\|\nabla P(w_0)\|^2\\
&=& \frac{b_H}{\gamma(m+1)}\|\nabla P(\widetilde{w}_{s-1})\|^2
\end{eqnarray*}

Hence, taking expectation, we obtain
\begin{eqnarray*}
\E[\|\nabla P(\widetilde{w}_s)\|^2]&\leq & \frac{b_H}{\gamma(m+1)}\E [\|\nabla P(\widetilde{w}_{s-1})\|^2] \\
& \leq & \biggl[ \frac{b_H}{\gamma(m+1)} \biggr]^s \|\nabla P(\widetilde{w}_0)\|^2
\end{eqnarray*}

\ifCLASSOPTIONcaptionsoff
  \newpage
\fi

%\begin{thebibliography}{1}
%
%\bibitem{IEEEhowto:kopka}
%H.~Kopka and P.~W. Daly, \emph{A Guide to \LaTeX}, 3rd~ed.\hskip 1em plus
%  0.5em minus 0.4em\relax Harlow, England: Addison-Wesley, 1999.
%
%\end{thebibliography}

\bibliography{ieee}

\begin{thebibliography}{10}
\bibitem{tieleman2017divide}
T.~Tieleman, G.~Hinton, Divide the gradient by a running average of its recent
  magnitude. coursera: Neural networks for machine learning, Technical Report.
\end{thebibliography}
\bibliographystyle{elsarticle-num}

\end{document}